\documentclass[11pt]{article}
\usepackage[preprint]{acl}
\usepackage{times}
\usepackage{latexsym}
\usepackage[T1]{fontenc}
\usepackage[utf8]{inputenc}
\usepackage{microtype}
\usepackage{graphicx}

\pdfoutput=1
\IfFileExists{inconsolata.sty}{\usepackage{inconsolata}}{}

\usepackage{amsmath,amsfonts,amssymb,amsthm,bm}

\usepackage{array,booktabs,multirow,tabularray,arydshln,hhline}
\usepackage{float,caption,subcaption}
\usepackage{enumitem}

\usepackage{xcolor,soul,todonotes}
\usepackage{tcolorbox}
\tcbuselibrary{skins}
\IfFileExists{bclogo.sty}{\usepackage[tikz]{bclogo}}{}
\usepackage{cleveref}

\newtheorem{assumptioninner}{Assumption}
\newcounter{assumption}[section]

\newtheorem{lemmainner}{Lemma}
\newcounter{lemma}[section]

\newtheorem{theoreminner}{Theorem}
\newcounter{theorem}[section]
\newenvironment{theorem}[1][]{
    \refstepcounter{theorem}
    \begin{theoreminner}[#1]
        
    }{
    \end{theoreminner}
}

\newtheorem{propositioninner}{Proposition}
\newcounter{proposition}[section]
\newenvironment{proposition}[1][]{
    \refstepcounter{proposition}
    \begin{propositioninner}[#1]
        
    }{
    \end{propositioninner}
}

\def\1{\bm{1}}

\def\rd{{\textnormal{d}}}

\def\ru{{\textnormal{u}}}

\def\rvd{{\mathbf{d}}}

\def\rvu{{\mathbf{i}}}

\def\rvu{{\mathbf{u}}}

\def\vzero{{\bm{0}}}
\def\vone{{\bm{1}}}

\def\vb{{\bm{b}}}
\def\vc{{\bm{c}}}
\def\vd{{\bm{d}}}
\def\ve{{\bm{e}}}
\def\vf{{\bm{f}}}

\def\vh{{\bm{h}}}

\def\vk{{\bm{k}}}
\def\vl{{\bm{l}}}

\def\vq{{\bm{q}}}
\def\vr{{\bm{r}}}
\def\vs{{\bm{s}}}

\def\vu{{\bm{u}}}
\def\vv{{\bm{v}}}
\def\vw{{\bm{w}}}
\def\vx{{\bm{x}}}
\def\vy{{\bm{y}}}
\def\vz{{\bm{z}}}

\def\mA{{\bm{A}}}
\def\mB{{\bm{B}}}
\def\mC{{\bm{C}}}
\def\mD{{\bm{D}}}

\def\mG{{\bm{G}}}
\def\mH{{\bm{H}}}
\def\mI{{\bm{I}}}
\def\mJ{{\bm{J}}}

\def\mL{{\bm{L}}}

\def\mP{{\bm{P}}}

\def\mR{{\bm{R}}}
\def\mS{{\bm{S}}}
\def\mT{{\bm{T}}}
\def\mU{{\bm{U}}}
\def\mV{{\bm{V}}}
\def\mW{{\bm{W}}}
\def\mX{{\bm{X}}}

\def\mZ{{\bm{Z}}}

\DeclareMathAlphabet{\mathsfit}{\encodingdefault}{\sfdefault}{m}{sl}
\SetMathAlphabet{\mathsfit}{bold}{\encodingdefault}{\sfdefault}{bx}{n}

\crefformat{section}{\S#2#1#3}
\crefformat{subsection}{\S#2#1#3}
\crefformat{appendix}{\S#2#1#3}

\def\eqref#1{equation~\ref{#1}}

\definecolor{c3}{cmyk}{0.3081,0,0.7209,0.3255}
\newtcbox{\hlprimarytab}{on line, rounded corners, box align=base, colback=c3!10,colframe=white,size=fbox,arc=3pt, before upper=\strut, top=-2pt, bottom=-4pt, left=-2pt, right=-2pt, boxrule=0pt}
\newtcbox{\hlsecondarytab}{on line, box align=base, colback=red!10,colframe=white,size=fbox,arc=3pt, before upper=\strut, top=-2pt, bottom=-4pt, left=-2pt, right=-2pt, boxrule=0pt}
\newcommand{\ua}[1]{\tiny{\hlprimarytab{($\uparrow$#1)}}}
\newcommand{\da}[1]{\tiny{\hlsecondarytab{($\downarrow$#1)}}}

\newcounter{findingcounter}
\renewcommand{\thefindingcounter}{\arabic{findingcounter}}

\makeatletter
\@ifundefined{HASBCLOGO}{%
  \newcommand{\finding}[1]{%
    \refstepcounter{findingcounter}%
    \begin{tcolorbox}[colback=black!10, colframe=black!20, boxrule=1pt,
                      arc=3pt, left=4pt, right=4pt, top=3pt, bottom=3pt,
                      drop shadow]
      {\em \small \textbf{F \thefindingcounter.} #1}
    \end{tcolorbox}%
  }%
}{%
  \newcommand{\finding}[1]{%
    \refstepcounter{findingcounter}%
    \begin{bclogo}[couleur=black!10, epBord=1, arrondi=0.1, logo=\bclampe, marge=2, ombre=true, blur, couleurBord=black!20, tailleOndu=3, sousTitre={\em \small \textbf{F \thefindingcounter.} #1}]{}
    \end{bclogo}%
  }%
}
\makeatother

\definecolor{kmycolor}{rgb}{0.858, 0.188, 0.478}

\newcommand{\scnorm}[1]{\left\lVert#1\right\rVert}

\title{Contextual Representation Flow Across Depth in \\State-Space and Transformer Architectures}

\author{
    Nhat M. Hoang$^{1}$, Do Xuan Long$^{2}$, Cong-Duy Nguyen$^{3}$, Chunyan Miao$^{1}$, \\
    \textbf{Min-Yen Kan}$^{2}$\textbf{, Luu Anh Tuan}$^{1,3}$\\
    $^{1}$Nanyang Technological University, $^{2}$National University of Singapore, $^{3}$VinUniversity\\
    \texttt{hoangmin003@e.ntu.edu.sg},
    \texttt{xuanlong.do@u.nus.edu},  
    \texttt{duy.ntc@vinuni.edu.vn}, \\
    \texttt{kanmy@comp.nus.edu.sg}, 
    \texttt{\{ascymiao, anhtuan.luu\}@ntu.edu.sg} \\
}

\begin{document}
\maketitle

\begin{abstract}
State Space Models (SSMs) have emerged as efficient, linear-time alternatives to Transformer-Based Models (TBMs) for long-context processing. However, their different mechanisms raise an important question: how does contextual representation flow across depth differ between these model families and their hybrids? We examine this question across twenty-one pretrained models, combining token-wise spectral diversity, layer-wise variation on seven datasets, and answer-location probing on two retrieval tasks. This joint analysis reveals that TBMs retain most of their token diversity at the output, whereas SSMs reach richer intermediate diversity but retain less of it on average, with late-layer compression. Comparisons with reinitialized weights reveal that these profiles depend on pretrained weights, while hybrids combine strong answer-location accessibility with architecture-dependent geometry, showing that high accessibility need not require the steadiest representation flow. Finally, we conduct theoretical case studies connecting these contrasts to mechanism-level differences: query-time token mixing in TBMs, repeated recurrent updates in SSMs, and preserved attention updates in hybrids.
\end{abstract}

\section{Introduction}
\label{sec:introduction}

The ability to process long inputs supports applications ranging from multi-turn dialogue to long-horizon autonomous agents, but remains a central challenge in language modeling \citep{Beltagy2020Longformer, liu2025comprehensive, pan2026natural}. Progress in processing these inputs has been driven largely by TBMs \citep{vaswani2017attention}, whose full-attention layers allow each token to access earlier tokens, but incur a computational cost quadratic in input length \citep{mamba1}. To address this cost, SSMs and related linear recurrences \citep{gu2022efficiently, mamba1, yang2025gated, peng2025rwkv7,lahoti2026mamba} efficiently compress past context in linear-time, and hybrid architectures combine this scalability with the expressive power of attention \citep{lieber2024jamba, de2024griffin, nvidia2025nemotronh}. However, how these architectural choices alter representation flow across depth remains unclear, leaving a major gap for optimal model selection, hybrid design, and downstream use.
% As these mechanisms operate differently, understanding representation flow across depth in these architectures matters for choosing, combining, and reading out from them.

Recent work has begun to probe the internal dynamics of these architectures \citep{fan2024not,ali2025hidden,men-etal-2025-shortgpt,3692070.3693514,ICLR2025_cbabc2f7,skean2025layer}. For example, \citet{skean2025layer} show that intermediate layers often outperform final layers for task-relevant information in both TBMs and SSMs. Concurrently, location probes also reveal that TBMs can encode where relevant information resides without utilizing it in their final outputs \citep{knowbutdonttell}. This distinction between internal encoding and final-output use motivates a joint analysis of representation geometry and answer-location accessibility across depth
% These findings highlight a critical decoupling between what models encode internally and what they use in final outputs, motivating a joint analysis of representation geometry and answer-location accessibility across depth.
This view also motivates studying attention--recurrence interactions in hybrids, where attention placement affects performance \citep{bae2026hybrid}.

% Recent representation analyses have started characterizing these architectures: spectral and geometric studies compare TBMs and SSMs and relate layer-wise representations to downstream embedding performance \citep{skean2024does, skean2025layer}, while location probes show that TBMs can encode where relevant information lies without using it in their answers \citep{knowbutdonttell}. These findings distinguish representation geometry from accessible task information, motivating a joint view of how both evolve across depth. This distinction also matters for hybrids: \citet{bae2026hybrid} find that early attention placement can hurt performance and hypothesize a conflict between global attention and Mamba's local inductive bias. Together, these results motivate a unified account of how token diversity, cross-layer change, and answer-location accessibility evolve during long-context processing, and how this representation flow differs between TBMs, SSMs, and their hybrids.

To examine these properties jointly, we present a unified study of how token diversity, cross-layer trajectory, and answer-location accessibility evolve across depth in TBMs, SSMs, and their hybrids. Our comparison covers twenty-one pretrained models and seven datasets spanning controlled retrieval, news, dialogue, and scientific and government documents. Across these datasets, our \emph{token-wise} analysis uses effective rank \citep{roy2007effective} to measure spectral diversity, while our \emph{layer-wise} analysis tracks coordinate variation across depth. Within this suite, key--value pair retrieval (KVPR) and multi-document question answering (MDQA) provide controlled answer positions \citep{liu-etal-2024-lost}, enabling \emph{probing} of whether a linear classifier can recover answer locations from representations. We compare pretrained and reinitialized weights to assess the contribution of learned weights within each architecture, then use theoretical case studies to characterize the token-mixing and update mechanisms that connect these properties. Our main contributions are:

\begin{itemize}[leftmargin=*, topsep=1pt]
    \item \textbf{Unified representation flow across depth.} We compare representation geometry and answer-location accessibility across depth in twenty-one TBM, SSM, and hybrid models.

    \item \textbf{Contrasting token diversity profiles.} On average, TBMs retain more of their peak token iversity at the output, whereas SSMs reach richer intermediate diversity but retain less of it; reinitialization reveals the contribution of learned weights to these profiles (Sec.~\ref{sssec:token-analysis}).

    \item \textbf{Layer-wise trajectory versus answer-location accessibility.} On average, SSMs follow steadier paths but TBMs expose answer locations more strongly; hybrids show that strong accessibility need not require the steadiest flow (Sec.~\ref{sssec:layer-analysis}--\ref{sssec:probing-analysis}).

    \item \textbf{Theoretical analysis of representation flow.} Conditional case studies link these contrasts to query-time token mixing in TBMs, repeated recurrent updates in SSMs, and preserved attention updates in hybrids (Sec.~\ref{sec:theoretical-analysis}).
\end{itemize}

\section{Related Work}
\label{sec:related-work}

% This architectural comparison builds on an efficiency--recall trade-off: SSMs and linear recurrences \citep{mamba1, mamba2, peng2025rwkv7, yang2025gated} offer efficiency but differ from TBMs in verbatim recall \citep{jelassi2024repeat, waleffe2024empirical}. Hybrids narrow this gap by combining attention with recurrence \citep{lieber2024jamba, de2024griffin, nvidia2025nemotronh, lee2025mambahybrid}. Internal differences between these mechanisms have been explored through causal interventions on synthetic recall tasks \citep{arora2026mechanistic} and systematic variation of attention placement \citep{bae2026hybrid}. We compare these mechanisms jointly through token diversity, cross-layer change and answer-location accessibility across depth.

Our comparison is situated within the efficiency--recall trade-off in long-context modeling. SSMs and linear recurrences \citep{mamba1, mamba2, peng2025rwkv7, lahoti2026mamba} scale linearly, but have limitations in verbatim associative recall compared to TBMs \citep{jelassi2024repeat, waleffe2024empirical}. Meanwhile, hybrid architectures seek to mitigate this vulnerability by strategically integrating attention mechanisms with recurrence \citep{lieber2024jamba, de2024griffin, nvidia2025nemotronh, lee2025mambahybrid}. These mechanisms have been studied through causal interventions on synthetic recall tasks \citep{arora2026mechanistic} and architectural ablations of attention placement \citep{bae2026hybrid}. Expanding beyond these targeted structural studies, we compare these mechanisms jointly through token diversity, cross-layer change and answer-location accessibility across depth.

Existing layer-wise representation studies show that intermediate layers often outperform final layers in both TBMs and SSMs \citep{skean2024does, skean2025layer}. Such intermediate-layer utility is also reflected in task-dependent early exit \citep{fan2024not}, while layer-pruning studies reveal substantial depth redundancy \citep{men-etal-2025-shortgpt}. These findings motivate a geometric analysis of how layers preserve and expose task-relevant information. To characterize this geometry, cosine similarity measures directional alignment \citep{ethayarajh-2019-contextual, ali2025hidden} and sharpness measures spatial dispersion \citep{wang2025understanding}, whereas effective rank summarizes the concentration of the full singular-value spectrum \citep{roy2007effective}. Its spectral view complements cosine analyses that can be dominated by a few directions \citep{timkey-vanschijndel-2021-bark, zhang2026oversmoothing}. Building on these complementary measures, we systematically analyze how contextual representations propagate and transform across depth.

\section{Empirical Analyses}
\label{sec:empircal-analysis}

We define representation flow as the evolution of token diversity, cross-layer change, and answer-location accessibility across depth. We measure the flow through \textbf{token-wise} analysis, change between layers through \textbf{layer-wise} analysis, and accessible answer locations through \textbf{probing}.

\subsection{Experimental Setup}
\label{ssec:setup}

\paragraph{Models.}
Our comparison covers twenty-one pretrained models spanning six TBMs, nine SSMs, and six hybrids. The TBMs comprise GPT-Neo-1.3B/-2.7B \citep{gpt-neo} Pythia-1.4B/-2.8B \citep{pythia}, Llama-3.1-8B-Instruct \citep{grattafiori2024llama3}, and Qwen2.5-32B Base \citep{qwen2025qwen25technicalreport}. The SSMs cover selective state space models Mamba-1.4B/-2.8B \citep{mamba1} and Mamba2-1.3B/-2.7B \citep{mamba2}, the linear recurrence RWKV-7-1.5B \citep{peng2025rwkv7}, the delta-rule models DeltaNet-1.3B \citep{yang2024delta} and Gated-DeltaNet-1.3B \citep{yang2025gated}, and Mamba-3-SISO/MIMO-1.5B \citep{lahoti2026mamba}. The hybrids combine recurrence with attention through Mamba2Attn-2.7B \citep{mamba2}, Zamba2-1.2B \citep{glorioso2024zamba2suitetechnicalreport}, RecurrentGemma-2B \citep{de2024griffin}, Qwen3.5-2B \citep{qwen2026qwen35}, Falcon-H1-1.5B \citep{falconh1}, and Nemotron-H-8B \citep{nvidia2025nemotronh}. Among these, GPT-Neo, Pythia, Mamba, Mamba2 and Mamba2Attn are pretrained on the Pile \citep{pile}, with the latter four sharing GPT-NeoX tokenization for a closer comparison. Qwen2.5-32B is substantially larger than the other models ($\le8$B parameters), but its $64$ layers provide a depth-matched TBM comparison with Mamba-2.8B, Mamba2-2.7B, and Mamba2Attn-2.7B.
Geometry summaries use all twenty-one models; probing family means use five TBMs, nine SSMs, and six hybrids, excluding Qwen2.5-32B while retaining its individual results. See Appendix~\ref{appdx:implementation} for more details about hybrid compositions.

\paragraph{Tasks.}
Following \citet{liu-etal-2024-lost}, we construct two controlled retrieval tasks: {Key--Value Pair Retrieval} (KVPR) and {Multi-Document Question Answering} (MDQA), where varying target answer positions yields ground-truth labels for probing. To evaluate representation flow beyond retrieval, we also include AG News \citep{NIPS2015_250cf8b5}, SAMSum \citep{gliwa-etal-2019-samsum}, CNN/DailyMail \citep{NIPS2015_afdec700,see-etal-2017-get}, Qasper \citep{dasigi-etal-2021-dataset}, and GovReport \citep{huang-etal-2021-efficient}. Geometry analyses use $2{,}000$-token inputs per model and dataset, matching the training context length of our Pile-trained models; probing uses complete retrieval prompts (Appendix~\ref{appdx:implementation}).
% Although absolute metric values in our token- and layer-wise analyses vary across datasets (Appendix~\ref{appdx:extra-datasets}), their representational trends remain consistent across inputs. We therefore report task-averaged results.

% To compare flow across input domains, token-wise and layer-wise analyses cover seven datasets: \emph{Key--Value Pair Retrieval} (KVPR), \emph{Multi-Document Question Answering} (MDQA), AG News \citep{NIPS2015_250cf8b5}, SAMSum \citep{gliwa-etal-2019-samsum}, CNN/DailyMail, Qasper \citep{dasigi-etal-2021-dataset} and GovReport \citep{huang-etal-2021-efficient}. KVPR and MDQA additionally provide the controlled answer positions needed for probing \citep{liu-etal-2024-lost}. Geometry uses $2$K-token samples per model and dataset, matching the training-context window of our Pile-trained comparison group. Tables average dataset-specific metrics equally over all seven datasets before averaging models within families; curves average the seven profiles at each native layer. Probes use task-specific prompts (Appendix~\ref{appdx:implementation}); dataset-wise geometry is in Appendix~\ref{appdx:extra-datasets}.

\paragraph{Implementation Details.} 
We evaluate pretrained models using Hugging Face Transformers \citep{wolf-etal-2020-transformers} and architecture-specific implementations for Mamba-3 and linear recurrences. For each model, we record $L+1$ token-representation matrices, where $\mH^{(0)}$ is the input embedding and $\mH^{(l)}$ is the output of Layer~$l$. The final state $\mH^{(L)}$ includes final layer normalization where applicable, whereas our theoretical model (Section~\ref{sec:theoretical-analysis}) evaluates the raw residual state $\mH_{\mathrm{raw}}^{(L)}$ prior to final normalization. Across all figures, layer depth is normalized to relative position $l/L \in [0,1]$. Appendix~\ref{appdx:implementation} gives implementation details, prompt templates, and probing hyperparameters.

\subsection{Token-wise Analysis}
\label{sssec:token-analysis}

We organize the token-wise analysis around two questions: (i) how does token diversity develop across depth and how much survives to the output layer across TBMs, SSMs, and hybrids? and (ii) which aspects of these family-specific profiles reflect learned pretraining dynamics rather than inherent architectural priors?

\paragraph{Setups.}
We quantify token diversity by the number of directions contributing meaningfully to the representation spectrum using effective rank \citep{roy2007effective}. Specifically, at layer $l$, $\mH^{(l)}\in\mathbb R^{n\times d}$ stacks $n$ token columns $\vh_i^{(l)}\in\mathbb R^d$ as transposed rows; $d$ is the hidden width. Its singular values $\sigma_1\ge\sigma_2\ge\cdots$ define normalized weights $p_k=\sigma_k/\sum_{k'}\sigma_{k'}$, yielding the effective rank:
\begin{equation}
\mathrm{eRank}(\mH^{(l)}) := \exp\Big(-\sum_{k=1}^{\min(n,d)} p_k \log p_k\Big),
\label{eq:erank}
\end{equation}
where higher values indicate greater spectral diversity. We report normalized rank $r^{(l)}=\mathrm{eRank}(\mH^{(l)})/\min(n,d)\in(0,1]$. We first average each dataset's samples, then summarize $r^{(l)}$ by its \emph{peak} (maximum over layers $l\ge1$) and \emph{retention} (final-layer value relative to peak). For fair comparison with prior studies (e.g., \citet{ali2025hidden}), we also calculate mean pairwise cosine similarity across distinct token positions and report results in Appendix~\ref{appdx:extra-datasets}. In addition, to isolate dependence on pretrained weights, we also explore reinitializing selected attention, recurrent-mixer, and MLP projections using Gaussian $\mathcal{N}(0,0.02^2)$, Xavier \citep{xavier}, or Kaiming \citep{he} initialization. Gaussian and Xavier also reinitialize SSM and hybrid token embeddings, while Kaiming leaves them pretrained (Appendix~\ref{appdx:implementation}).

\begin{figure*}[t!]
\centering
\includegraphics[width=\textwidth]{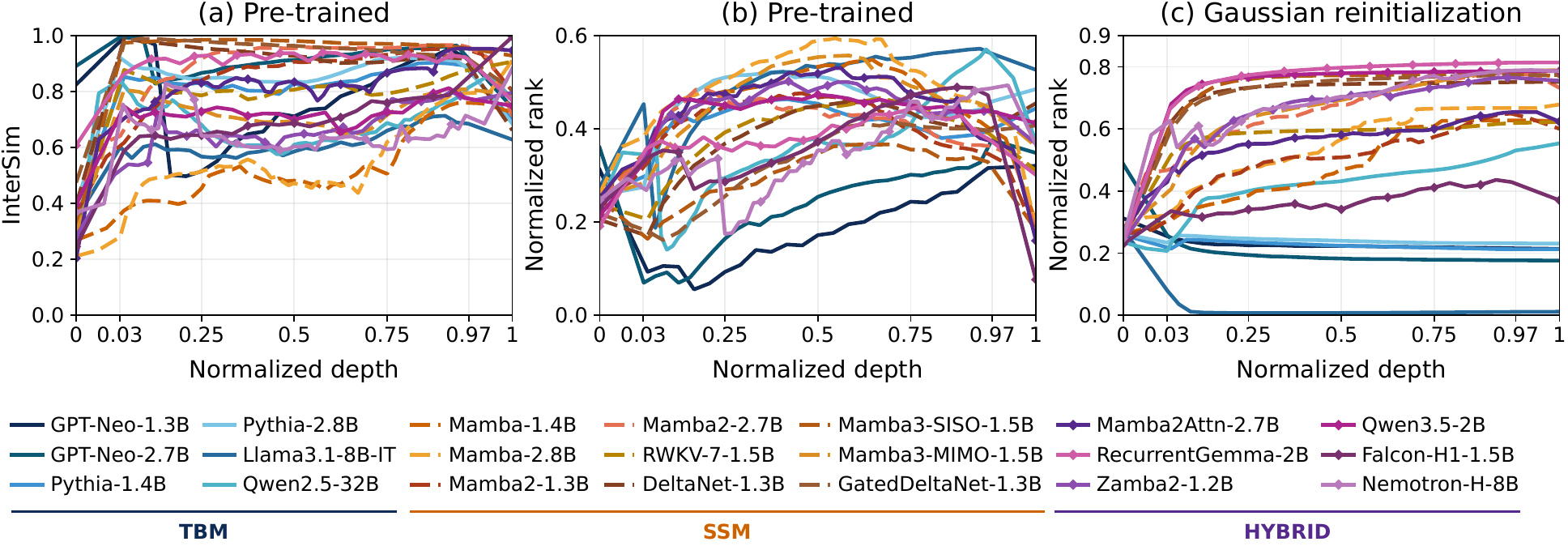}
\caption{Token geometry of all twenty-one models, averaging seven datasets at each native layer: (a) pretrained InterSim, (b) pretrained normalized rank, and (c) Gaussian-reinitialized rank.}
\label{fig:rank-vs-intersim}
\end{figure*}

% \paragraph{Cosine similarity hides changes in spectral diversity.}
% Before comparing model families, we examine prior diagnoses of over-smoothing based on inter-token cosine similarity \citep{ali2025hidden} and normalized pairwise distances \citep{wang2025understanding}. In Mamba2-1.3B's seven-dataset mean profile, normalized rank drops from $0.477$ at Layer~22 to $0.175$ at the output, while InterSim decreases from $0.892$ to $0.802$, still indicating strong directional alignment (Figure~\ref{fig:rank-vs-intersim}(a,b)). Pythia-1.4B and Mamba-1.4B likewise have similar mean final InterSim despite more than a twofold difference in rank. Mean cosine can thus miss substantial spectral differences, consistent with its known limitations \citep{timkey-vanschijndel-2021-bark}.

% The distance-based metric of \citet{wang2025understanding} offers a complementary perspective. Their \emph{sharpness} at one layer, $E(\mH)=\sum_{i\ne j}\|\vh_i-\vh_j\|^2/\big[2(n-1)\sum_i\|\vh_i\|^2\big]$, evaluates normalized pairwise squared distances. Mamba-1.4B has lower mean sharpness retention but higher mean peak normalized rank than size-matched Pythia-1.4B across seven datasets (Table~\ref{tab:adjudication}). Lower output diversity can thus follow richer intermediate spectra, motivating effective rank for our depthwise comparison.

\begin{table}[t]
\centering
\footnotesize
\setlength{\tabcolsep}{2.5pt}
\resizebox{\columnwidth}{!}{
\begin{tabular}{l cc cccc}
\toprule
 & \multicolumn{2}{c}{Rank$\uparrow$} & \multicolumn{4}{c}{Final rank$\uparrow$} \\
\cmidrule(lr){2-3}\cmidrule(lr){4-7}
Model & Peak & Ret.$\uparrow$ & Pre. & Gau. & Xav. & Kai. \\
\midrule
GPT-Neo-1.3B & 0.329 & 96.5 & 0.321 & 0.214 & 0.086 & 0.100 \\
GPT-Neo-2.7B & 0.383 & 90.1 & 0.349 & 0.176 & 0.081 & 0.092 \\
Pythia-1.4B & 0.470 & 93.3 & 0.444 & 0.213 & 0.123 & 0.162 \\
Pythia-2.8B & 0.525 & 91.4 & 0.486 & 0.231 & 0.130 & 0.159 \\
Llama-3.1-8B-Instruct & 0.576 & 89.5 & 0.527 & 0.012 & 0.007 & 0.010 \\
Qwen2.5-32B Base & 0.571 & 63.2 & 0.372 & 0.554 & 0.009 & 0.075 \\
\textit{Mean} & \textit{0.476} & \textit{87.3} & \textit{0.416} & \textit{0.233} & \textit{0.073} & \textit{0.100} \\
\midrule
Mamba-1.4B & 0.558 & 35.4 & 0.201 & 0.625 & 0.562 & 0.669 \\
Mamba-2.8B & 0.606 & 30.1 & 0.186 & 0.679 & 0.591 & 0.702 \\
Mamba2-1.3B & 0.478 & 35.3 & 0.175 & 0.601 & 0.494 & 0.671 \\
Mamba2-2.7B & 0.488 & 36.0 & 0.180 & 0.732 & 0.652 & 0.764 \\
RWKV-7-1.5B & 0.466 & 62.8 & 0.311 & 0.623 & 0.600 & 0.363 \\
DeltaNet-1.3B & 0.479 & 93.1 & 0.455 & 0.752 & 0.700 & 0.749 \\
Mamba-3 SISO-1.5B & 0.371 & 48.4 & 0.184 & 0.770 & 0.780 & 0.641 \\
Mamba-3 MIMO-1.5B & 0.559 & 63.7 & 0.363 & 0.789 & 0.803 & 0.675 \\
Gated-DeltaNet-1.3B & 0.431 & 92.2 & 0.407 & 0.774 & 0.710 & 0.733 \\
\textit{Mean} & \textit{0.493} & \textit{55.2} & \textit{0.274} & \textit{0.705} & \textit{0.654} & \textit{0.663} \\
\midrule
Mamba2Attn-2.7B & 0.536 & 28.9 & 0.160 & 0.623 & 0.502 & 0.688 \\
Zamba2-1.2B & 0.506 & 70.4 & 0.369 & 0.756 & 0.717 & 0.773 \\
Qwen3.5-2B & 0.489 & 82.9 & 0.413 & 0.786 & 0.794 & 0.778 \\
RecurrentGemma-2B & 0.448 & 65.3 & 0.301 & 0.814 & 0.837 & 0.784 \\
Falcon-H1-1.5B & 0.495 & 15.2 & 0.077 & 0.370 & 0.353 & 0.396 \\
Nemotron-H-8B & 0.494 & 73.8 & 0.366 & 0.786 & 0.597 & 0.340 \\
\textit{Mean} & \textit{0.495} & \textit{56.1} & \textit{0.281} & \textit{0.689} & \textit{0.633} & \textit{0.627} \\
\bottomrule
\end{tabular}
}
\caption{Token diversity averaged over all seven datasets: peak normalized rank, rank retention (\% of peak), final rank before and after reinitialization. Peaks and ratios are computed within each dataset before averaging; \emph{Mean} rows then average models within families.}
\label{tab:adjudication}
\end{table}

\paragraph{SSMs reach higher mean peak diversity, TBMs retain more at output, and hybrids show intermediate retention.}
Evaluating token diversity across depth separates reaching peak diversity from preserving it to the output layer (Table~\ref{tab:adjudication}). Averaged across seven datasets, SSMs achieve higher mean peak rank than TBMs ($0.493$ vs.\ $0.476$), but retain less at the output ($55.2\%$ vs.\ $87.3\%$). The peak ordering holds on six datasets, with MDQA reversing it, whereas TBMs retain a larger fraction of peak diversity on all seven (Appendix~\ref{appdx:extra-datasets}). Hybrids exhibit intermediate mean retention ($56.1\%$), though as shown in Figure~\ref{fig:rank-vs-intersim}(b), this aggregate conceals architecture-specific profiles. Most SSMs and hybrids build rich intermediate representations before losing diversity toward the output while TBMs generally maintain or recover their diversity across depth. For instance, DeltaNet-1.3B and Gated-DeltaNet-1.3B largely preserve their peaks, Pythia diversifies in early layers while GPT-Neo recovers after an early rank minimum at Layer~4.

\finding{
While SSMs achieve higher mean peak token diversity overall, TBMs retain a larger fraction of this peak at the output layer ($87.3\%$ vs.\ $55.2\%$), and hybrids show intermediate mean retention ($56.1\%$), with model means spanning $15.2$--$82.9\%$.
\label{F1}
}

\paragraph{Pretrained and reinitialized weights yield opposite family contrasts.}
\label{sssec:untrained-priors}
To assess how this compression depends on pretrained weights, we evaluate models under reinitialized weight schemes. As illustrated in Figure~\ref{fig:rank-vs-intersim}(c), Gaussian-reinitialized SSMs and hybrids generally reach higher output ranks than TBMs. Across all three schemes, their seven-dataset family means are $3.0$--$9.0\times$ the TBM mean (Table~\ref{tab:adjudication}). Comparing pretrained models against these reinitialized baselines reveals opposite family-average contrasts: pretrained TBMs have broader output spectra, while SSMs and hybrids have narrower spectra. Individual models depart from this pattern, such as Qwen2.5-32B under Gaussian, RWKV-7 under Kaiming reinitialization, and Nemotron-H under Kaiming on six datasets. A depth-matched 24-layer control supports the family contrast (Appendix~\ref{appdx:implementation}), showing that recurrence need not enforce rank compression; the family-average compression depends on pretrained weights.

\finding{
Across seven datasets, reinitialized SSM and hybrid family means are $3.0$--$9.0\times$ the TBM mean output rank; pretrained TBMs have broader output spectra while SSMs and hybrids have narrower spectra than their reinitialized baselines on average.\label{F2}
}

\subsection{Layer-wise Analysis}
\label{sssec:layer-analysis}

Shifting from within-layer spectra to cross-layer representation flow, we address two questions: (i) how smoothly do representations transform across consecutive layers in TBMs versus SSMs? and (ii) do hybrids adopt a intermediate cross-layer trajectory or exhibit model-specific variation?

\paragraph{Setups.}
Let $u_q^{(l)}$ denote the $q$-th scalar entry of representation $\mH^{(l)}$ across $nd$ total dimensions. \emph{Local variation} measures how far each layer $\mH^{(l)}$ departs from the midpoint of its immediate neighbors:
\begin{equation}
V_{\mathrm{loc}}=\frac{1}{2nd(L-1)}\sum_{q,l}\big|u_q^{(l-1)}-2u_q^{(l)}+u_q^{(l+1)}\big|,
\end{equation}
summed over entries $q$ and inner layers $l$. $V_{\mathrm{loc}}$ evaluates discrete second differences, serving as the depth analogue of trajectory curvature in models \citep{henaff2019perceptual, hosseini2023large}, where small values reflect residual networks acting as iterative refiners \citep{greff2017highway, jastrzebski2018residual, sander2022residual}. In parallel, \emph{global variation} measures total coordinate dispersion across all layers:
\begin{equation}
V_{\mathrm{glob}}=\frac{1}{nd}\sum_{q=1}^{nd}
\sqrt{\frac{1}{L+1}\sum_{l=0}^{L}\big(u_q^{(l)}-\bar u_q\big)^2},
\end{equation}
where $\bar u_q$ is the coordinate mean across layers. Both metrics scale with activation magnitude, so their ratio $\rho=V_{\mathrm{loc}}/V_{\mathrm{glob}}$ cancels common scaling when $V_{\mathrm{glob}}>0$. For $L\ge3$, $\rho$ satisfies $0\le\rho\le C_L\le\sqrt{5}$, giving models a bounded scale without additional normalization, where $C_L=\sqrt{(L+1)(4L-7)}/(L-1)$ is the sharper depth-specific bound (Proposition~\ref{prop:variation-range}). Table~\ref{tab:variation} computes all three measures across each model's $L+1$ native states, including embedding and final normalized output. Equal consecutive updates yield $\rho=0$, with lower ratios indicating less local variation relative to global dispersion (Proposition~\ref{prop:increments}).

\begin{table}[t!]
\centering
\fontsize{8}{9.5}\selectfont
\setlength{\tabcolsep}{1.5pt}
\resizebox{\columnwidth}{!}{
\begin{tabular}{r l l ccc}
\toprule
$L$ & Family & Model & $V_{\mathrm{loc}}$ & $V_{\mathrm{glob}}$ & $\rho\downarrow$ \\
\midrule
24 & TBM & GPT-Neo-1.3B & 1.01 & 3.12 & 0.324 \\
24 & TBM & Pythia-1.4B & 0.28 & 0.96 & 0.289 \\
24 & SSM & RWKV-7-1.5B & 0.67 & 2.61 & 0.257 \\
24 & SSM & DeltaNet-1.3B & 1.11 & 4.56 & 0.244 \\
24 & SSM & Mamba-3 SISO-1.5B & 2.84 & 11.33 & 0.247 \\
24 & SSM & Mamba-3 MIMO-1.5B & 0.53 & 2.17 & 0.245 \\
24 & SSM & Gated-DeltaNet-1.3B & 0.94 & 3.56 & 0.263 \\
24 & Hybrid & Qwen3.5-2B & 0.09 & 0.51 & 0.180 \\
24 & Hybrid & Falcon-H1-1.5B & 0.93 & 4.78 & 0.194 \\
26 & Hybrid & RecurrentGemma-2B & 1.88 & 5.90 & 0.318 \\
32 & TBM & GPT-Neo-2.7B & 0.88 & 3.34 & 0.263 \\
32 & TBM & Pythia-2.8B & 0.27 & 1.13 & 0.239 \\
32 & TBM & Llama-3.1-8B-Instruct & 0.07 & 0.36 & 0.194 \\
38 & Hybrid & Zamba2-1.2B & 0.14 & 0.83 & 0.169 \\
48 & SSM & Mamba-1.4B & 0.13 & 0.79 & 0.162 \\
48 & SSM & Mamba2-1.3B & 1.36 & 7.93 & 0.172 \\
52 & Hybrid & Nemotron-H-8B & 0.25 & 1.28 & 0.195 \\
64 & TBM & Qwen2.5-32B Base & 0.75 & 3.66 & 0.205 \\
64 & SSM & Mamba-2.8B & 0.12 & 0.92 & 0.133 \\
64 & SSM & Mamba2-2.7B & 1.39 & 9.42 & 0.148 \\
64 & Hybrid & Mamba2Attn-2.7B & 0.98 & 7.01 & 0.140 \\
\midrule
-- & TBM & \textit{Mean} & \textit{0.54} & \textit{2.09} & \textit{0.252} \\
-- & SSM & \textit{Mean} & \textit{1.01} & \textit{4.81} & \textit{0.208} \\
-- & Hybrid & \textit{Mean} & \textit{0.71} & \textit{3.38} & \textit{0.199} \\
\bottomrule
\end{tabular}
}
\caption{Models ordered by native depth $L$, then family. All measures include embedding and final normalization; $\rho$ averages per-input ratios, then seven datasets equally. Family means span all listed depths.}
\label{tab:variation}
\end{table}

\paragraph{SSMs follow steadier cross-layer paths than TBMs.}
Depth-matched comparisons retain the TBM--SSM steadiness contrast (Table~\ref{tab:variation}). At $L=24$, seven-dataset mean variation ratios are lower for SSMs ($\rho=0.244$--$0.263$) than TBMs ($0.289$--$0.324$), with hybrids lower still ($0.180$--$0.194$). At $L=64$, Mamba, Mamba2, and Mamba2Attn likewise have lower ratios ($0.133$--$0.148$) than the much larger Qwen2.5-32B ($0.205$), although the hybrid falls within the SSM range.

\finding{
Across seven datasets, SSM mean native-layer variation ratios are $0.82\times$ TBM means, indicating steadier cross-layer trajectories on average.
\label{F3}
}

\paragraph{Hybrid paths range from steadier than SSMs to TBM-like variation.}
Hybrid architectures do not occupy a uniform middle ground in trajectory steadiness, exhibiting a much broader spread than their token diversity profiles alone would suggest. Their overall family mean ($\rho=0.199$) is lower than both SSMs ($\rho=0.208$) and TBMs ($\rho=0.252$), but individual hybrid profiles diverge sharply: Mamba2Attn achieves the lowest ratio among all hybrids ($\rho=0.140$), followed closely by Zamba2 and Qwen3.5, whereas RecurrentGemma exhibits a TBM-like variation ($\rho=0.318$), exceeding its $26$-layer random-walk baseline. These architecture-specific trajectories motivate testing whether cross-layer steadiness aligns with the decodability of task information across depth.

\finding{
Hybrid trajectories range from steadier than SSMs to TBM-like variation, rather than forming a uniform intermediate regime.
\label{F4}
}

\subsection{Probing Analysis}
\label{sssec:probing-analysis}

To connect geometric trajectories to decodable answer locations, we evaluate linear probes across depth and context positions to address three questions: (i) at what depth does task accessibility peak, and how much persists to the output layer? (ii) how do positional retrieval profiles contrast across model families? and (iii) does task decodability track token diversity or trajectory steadiness?

\paragraph{Setups.}
To measure decodable information after processing context and query, we probe the final prompt token representation $\vh_{\mathrm{last}}^{(l)}\in\mathbb{R}^d$ at each layer $l$. An affine probe $\vf^{(l)}(\vh)=\mW^{(l)}\vh+\vb^{(l)}$ with learned weights $\mW^{(l)}$ and biases $\vb^{(l)}$ is trained to minimize multiclass cross-entropy loss:
\begin{equation}
\mathcal{L}^{(l)} = -\frac{1}{M} \sum_{i=1}^M \sum_{c=1}^C y_{i,c} \log \hat{y}_{i,c}^{(l)},
\end{equation}
where $M$ is the number of training instances, $C$ is the number of target position classes ($C=11$ for KVPR, $C=7$ for MDQA), $y_{i,c}$ is the ground-truth binary label indicator, and $\hat{y}_{i,c}^{(l)} = \mathrm{softmax}(\vf^{(l)}(\vh_{\mathrm{last}}^{(l)}))_c$. 

The probe is trained to predict the queried key--value pair's position among 11 sampled target locations in 35 KVPR pairs (chance baseline: $9.1\%$), or the answer-bearing document among seven target positions ($1,3,\ldots,13$) in 13 MDQA documents (chance baseline: $14.3\%$). For GPT-Neo and Pythia, final-token representations match or exceed mean-pooled representations (Table~\ref{tab:mean-token-probing}). We report \emph{peak} accuracy (maximum seed-mean across layers) and the \emph{decay gap} $\Delta$ from peak to final layer. Group-disjoint data splits ($70/10/20$), validation-selected weights, and permuted-label controls are detailed in Appendix~\ref{appdx:implementation}.

\begin{table*}[t!]
\centering
\renewcommand{\ua}[1]{\tcbox[nobeforeafter,tcbox raise base,colback=c3!10,colframe=white,boxrule=0pt,arc=1pt,boxsep=0pt,left=1.5pt,right=1.5pt,top=0pt,bottom=0pt]{$#1$}}
\renewcommand{\da}[1]{\tcbox[nobeforeafter,tcbox raise base,colback=red!10,colframe=white,boxrule=0pt,arc=1pt,boxsep=0pt,left=1.5pt,right=1.5pt,top=0pt,bottom=0pt]{$#1$}}
\small
\setlength{\tabcolsep}{2.2pt}
\resizebox{\textwidth}{!}{
\begin{tabular}{l cccr cccr}
\toprule
 & \multicolumn{4}{c}{KVPR} & \multicolumn{4}{c}{MDQA} \\
\cmidrule(lr){2-5} \cmidrule(lr){6-9}
Model & Peak$\uparrow$ (layer) & Final$\uparrow$ & $\Delta\downarrow$ & $S_{\mathrm{rank}}$ & Peak$\uparrow$ (layer) & Final$\uparrow$ & $\Delta\downarrow$ & $S_{\mathrm{rank}}$ \\
\midrule
GPT-Neo-1.3B & $89.7 \pm 0.6$ (12/24) & $62.1 \pm 0.8$ & $27.6$ & \ua{+0.54} & $65.0 \pm 1.6$ (9/24) & $29.3 \pm 1.6$ & $35.7$ & \da{-0.09} \\
GPT-Neo-2.7B & $82.1 \pm 0.9$ (16/32) & $61.1 \pm 0.9$ & $21.0$ & \ua{+0.56} & $73.4 \pm 1.1$ (13/32) & $38.9 \pm 2.0$ & $34.5$ & \ua{+0.16} \\
Pythia-1.4B & $99.3 \pm 0.2$ (10/24) & $82.8 \pm 1.6$ & $16.5$ & \ua{+0.64} & $62.9 \pm 3.2$ (9/24) & $39.4 \pm 0.5$ & $23.5$ & \ua{+0.62} \\
Pythia-2.8B & $99.8 \pm 0.1$ (15/32) & $85.9 \pm 0.7$ & $13.8$ & \ua{+0.85} & $63.9 \pm 2.5$ (11/32) & $45.1 \pm 2.1$ & $18.8$ & \ua{+0.81} \\
Llama-3.1-8B-Instruct & $100.0 \pm 0.0$ (10/32) & $99.8 \pm 0.0$ & $0.2$ & \ua{+0.31} & $75.8 \pm 3.2$ (15/32) & $51.9 \pm 5.0$ & $23.9$ & \ua{+0.41} \\
Qwen2.5-32B Base & $100.0 \pm 0.0$ (45/64) & $99.7 \pm 0.1$ & $0.3$ & \ua{+0.35} & $82.0 \pm 3.3$ (55/64) & $77.0 \pm 3.2$ & $4.9$ & \ua{+0.83} \\
\textit{Mean} & \textit{94.2} & \textit{78.3} & \textit{15.8} &  & \textit{68.2} & \textit{40.9} & \textit{27.3} &  \\
\midrule
Mamba-1.4B & $59.0 \pm 1.3$ (27/48) & $44.9 \pm 0.9$ & $14.1$ & \da{-0.32} & $44.3 \pm 1.4$ (24/48) & $29.8 \pm 0.6$ & $14.5$ & \ua{+0.94} \\
Mamba-2.8B & $70.7 \pm 0.5$ (41/64) & $47.1 \pm 1.4$ & $23.6$ & \da{-0.30} & $46.5 \pm 1.8$ (36/64) & $32.8 \pm 1.3$ & $13.6$ & \ua{+0.49} \\
Mamba2-1.3B & $67.0 \pm 1.0$ (23/48) & $44.9 \pm 1.2$ & $22.1$ & \ua{+0.05} & $42.7 \pm 1.5$ (12/48) & $30.8 \pm 2.3$ & $12.0$ & \ua{+0.86} \\
Mamba2-2.7B & $79.8 \pm 0.4$ (26/64) & $55.1 \pm 0.5$ & $24.7$ & \da{-0.34} & $44.6 \pm 1.1$ (10/64) & $30.0 \pm 0.9$ & $14.6$ & \ua{+0.95} \\
RWKV-7-1.5B & $85.4 \pm 1.0$ (14/24) & $57.5 \pm 0.8$ & $27.9$ & \ua{+0.74} & $47.1 \pm 1.7$ (6/24) & $29.9 \pm 0.5$ & $17.1$ & \ua{+0.17} \\
DeltaNet-1.3B & $45.1 \pm 1.1$ (15/24) & $43.3 \pm 1.1$ & $1.8$ & \ua{+0.34} & $39.5 \pm 1.5$ (10/24) & $24.5 \pm 1.5$ & $15.1$ & \da{-0.13} \\
Mamba-3 SISO-1.5B & $61.6 \pm 1.3$ (13/24) & $54.2 \pm 1.0$ & $7.5$ & \ua{+0.62} & $43.5 \pm 1.6$ (12/24) & $26.7 \pm 1.6$ & $16.8$ & \ua{+0.11} \\
Mamba-3 MIMO-1.5B & $71.1 \pm 0.8$ (15/24) & $54.4 \pm 0.3$ & $16.7$ & \ua{+0.45} & $39.8 \pm 1.0$ (11/24) & $28.6 \pm 1.1$ & $11.2$ & \ua{+0.70} \\
Gated-DeltaNet-1.3B & $45.5 \pm 0.5$ (15/24) & $44.2 \pm 0.8$ & $1.3$ & \ua{+0.83} & $40.0 \pm 1.1$ (11/24) & $27.0 \pm 1.5$ & $13.1$ & \ua{+0.12} \\
\textit{Mean} & \textit{65.0} & \textit{49.5} & \textit{15.5} &  & \textit{43.1} & \textit{28.9} & \textit{14.2} &  \\
\midrule
Mamba2Attn-2.7B & $99.9 \pm 0.1$ (37/64) & $90.8 \pm 0.6$ & $9.2$ & \da{-0.13} & $58.8 \pm 2.3$ (38/64) & $39.4 \pm 2.2$ & $19.4$ & \ua{+0.94} \\
Zamba2-1.2B & $99.3 \pm 0.2$ (19/38) & $96.9 \pm 0.7$ & $2.4$ & \da{-0.10} & $58.5 \pm 2.9$ (20/38) & $37.0 \pm 1.3$ & $21.5$ & \ua{+0.43} \\
Qwen3.5-2B & $97.2 \pm 0.3$ (16/24) & $87.3 \pm 0.8$ & $9.9$ & \da{-0.21} & $65.4 \pm 3.2$ (15/24) & $53.2 \pm 3.7$ & $12.2$ & \da{-0.19} \\
RecurrentGemma-2B & $99.2 \pm 0.1$ (18/26) & $89.0 \pm 0.9$ & $10.3$ & \ua{+0.63} & $52.5 \pm 2.4$ (14/26) & $29.8 \pm 1.4$ & $22.8$ & \ua{+0.51} \\
Falcon-H1-1.5B & $100.0 \pm 0.0$ (15/24) & $99.6 \pm 0.1$ & $0.4$ & \ua{+0.86} & $58.9 \pm 2.6$ (12/24) & $42.3 \pm 3.2$ & $16.6$ & \ua{+0.22} \\
Nemotron-H-8B & $98.6 \pm 0.2$ (29/52) & $93.6 \pm 0.6$ & $5.0$ & \ua{+0.45} & $66.5 \pm 2.2$ (30/52) & $40.5 \pm 1.2$ & $25.9$ & \ua{+0.31} \\
\textit{Mean} & \textit{99.0} & \textit{92.9} & \textit{6.2} &  & \textit{60.1} & \textit{40.4} & \textit{19.7} &  \\
\bottomrule
\end{tabular}
}
\caption{Location-probe accuracy (\%, mean and $95\%$ interval, five seeds; $\uparrow$), peak-to-final drop ($\Delta\downarrow$), and Spearman rank correlations ($S_{\mathrm{rank}}$; green/red: positive/negative). Qwen2.5-32B is excluded from probing family means.}
\label{tab:probing-regen}
\end{table*}

\paragraph{TBMs and hybrids outperform SSMs in positional accessibility, especially in the middle.}
Across all $42$ reported model--task pairs, probing accuracy peaks before the output layer, occurring between one-third and two-thirds depth in $35$ cases (Table~\ref{tab:probing-regen}). As shown in Figure~\ref{fig:per-position}, the TBM--SSM performance gap is concentrated at central context positions, mirroring the "lost-in-the-middle" phenomenon in TBMs \citep{knowbutdonttell}. While all families achieve over $90\%$ mean accuracy at final context positions except hybrids on MDQA ($87.5\%$), SSM central accuracy drops to $40.7\%$ on KVPR (versus $88.2\%$ for TBMs and $98.2\%$ for hybrids) and $15.2\%$ on MDQA (versus $51.7\%$ for TBMs and $43.2\%$ for hybrids). TBMs expose central answer locations far better, yielding higher mean peak ($94.2\%$ KVPR, $68.2\%$ MDQA) and final-layer accuracy ($78.3\%$ KVPR, $40.9\%$ MDQA) than SSMs ($65.0\%/49.5\%$ KVPR, $43.1\%/28.9\%$ MDQA). Hybrids largely eliminate the SSM central deficit on KVPR, achieving at least $97.3\%$ mean accuracy at every sampled position, while reaching $43.2\%$ at the MDQA midpoint. Their mean final-layer accuracy leads on KVPR ($92.9\%$), while on MDQA it is $40.4\%$, only $0.55$ percentage points below TBMs.

\begin{figure}[t!]
\centering
\includegraphics[width=\columnwidth]{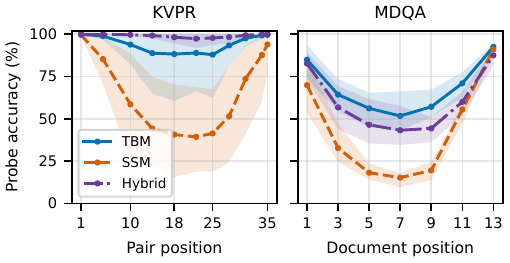}
\caption{Peak-layer probing accuracy by item position: family means and model ranges exclude Qwen2.5-32B. MDQA targets occupy positions $1,3,\ldots,13$.}
\label{fig:per-position}
\end{figure}

\finding{
TBM and hybrid mean central-position probe accuracies are $2.2$--$3.4\times$ and $2.4$--$2.8\times$ the SSM mean, respectively, across both benchmarks.
\label{F5}
}

\paragraph{Task decodability decouples from overall token spectral rank.}
Comparing Pile-matched 2.7B models highlights that strong linear task decodability does not require high token spectral diversity. On KVPR, Mamba2Attn reaches $90.8\%$ final probe accuracy versus Mamba2's $55.1\%$, despite both models exhibiting nearly identical, low task-specific rank retention ($11.7\%$ vs.\ $10.0\%$). Mamba2Attn similarly achieves higher final MDQA accuracy despite lower rank retention (Tables~\ref{tab:mamba2-task-geometry}, \ref{tab:probing-regen}). Geometrically, a fixed affine probe for $K$ target positions reads at most $K-1$ independent score-difference directions across prompt instances, whereas effective rank measures full singular-value dispersion among all token vectors within a single sequence (Proposition~\ref{prop:score}). Within-model rank--accuracy correlations ($S_{\mathrm{rank}}$) are uniformly positive for TBMs on KVPR, mixed for hybrids on KVPR, and predominantly positive otherwise. Representation geometry and task accessibility capture distinct properties (Appendix~\ref{appdx:hybrid-accessibility}).

\finding{
Despite intermediate mean rank retention, hybrids competitively lead family-mean final probe accuracy, exceeding the SSM mean by $1.40$--$1.88\times$.
\label{F6}}

\section{Theoretical Analysis}
\label{sec:theoretical-analysis}

Our empirical findings establish three core representational contrasts: TBMs preserve token diversity across depth and better expose central answer locations ($F\ref{F1}, F\ref{F5}$), SSMs follow steadier cross-layer paths while undergoing late-layer rank compression ($F\ref{F1}, F\ref{F3}$), and hybrids decouple strong answer accessibility from broad token spectra ($F\ref{F1}, F\ref{F6}$). To ground these empirical patterns in architectural mechanics, we conduct a theoretical analysis about single-head causal self-attention in TBMs, structured state-space duality (SSD) recurrence in SSMs \citep{mamba2}, and their integration in hybrid architectures. Our complete proofs are provided in Appendix~\ref{appdx:theory-proofs}.

\subsection{Sequence Mixing Mechanisms}
To ground empirical differences in architectural constraints, we formalize how models route sequence context. Attention combines value vectors $\vv_j$ from current and earlier tokens based on how their keys $\vk_j$ match the current query $\vq_t$ \citep{vaswani2017attention}. For head width $p$, its weights are $P_{t,j}=\operatorname{softmax}_{j\le t}(\vq_t^\top\vk_j/\sqrt p)$, normalized to sum to one, and its output is $\sum_{j\le t}P_{t,j}\vv_j$. Here $t,j$ index token positions; $j\le t$ enforces causal access. Mamba2 instead writes a token's convolved features $\vx_t\in\mathbb R^p$ into a state $\mS_t\in\mathbb R^{N\times p}$, with $N$ state coordinates per feature channel:
\begin{equation}\label{eq:ssd-main}
\begin{aligned}
 \mS_t&=a_t\mS_{t-1}+\Delta_t\vb_t\vx_t^\top,\\
 \vy_t&=\mS_t^\top\vc_t+D\vx_t,\quad a_t=e^{-\lambda\Delta_t}.
\end{aligned}
\end{equation}
where $\vy_t$ is the head output, $\vb_t,\vc_t\in\mathbb R^N$ control state writing and reading, $\Delta_t>0$ governs step-wise writing and forgetting, $\lambda>0$ is a learned decay rate, $D$ scales a direct input skip connection, and $\mS_0=\bm{0}$. Holding input-dependent coefficients fixed, the weight of input $\vx_j$ in the state read at position $t$ is
\begin{equation}\label{eq:ssd-kernel-main}
 T_{t,j}=\Delta_j\langle\vc_t,\vb_j\rangle
                \prod_{u=j+1}^{t}a_u,\qquad j\le t.
\end{equation}
The inner product aligns writing with reading, while the product accumulates exponential decay. Splitting the sequence after position $c$, factoring this scalar product through the $N$ state coordinates yields a matrix rank bound: $\operatorname{rank}\mT_{t>c,j\le c}\le N$ for earlier-to-later mixing weights. Distinct query--key matches in attention can yield mixing matrices across the same split with rank exceeding $N$, whereas a single SSD head has at most $N$ independent temporal mixing modes when its realized coefficients are held fixed. This provides an architectural perspective on positional accessibility (Finding~\ref{F5}).

\subsection{Recurrent Layer Dynamics}
We construct a recurrent-update mechanism under which repeated application contracts token diversity. To formalize this mechanism, we evaluate changes between consecutive layers. The recorded layer state $\mH^{(l)}$ equals the raw residual state $\mH_{\mathrm{raw}}^{(l)}$ for all inner layers $l<L$, with final output $\mH^{(L)}=\mathcal N_f(\mH_{\mathrm{raw}}^{(L)})$ applying final layer normalization $\mathcal N_f$ where present. The block update matrix $\mU^{(l)}=F_l(\mH_{\mathrm{raw}}^{(l)})$ incorporates internal layer normalizations, yielding residual dynamics:
\begin{equation}\label{eq:ssd-drift-main}
\begin{aligned}
 \mH_{\mathrm{raw}}^{(l+1)}&=\mH_{\mathrm{raw}}^{(l)}+\mU^{(l)},\\
 \mH_{\mathrm{raw}}^{(l+2)}-2\mH_{\mathrm{raw}}^{(l+1)}+\mH_{\mathrm{raw}}^{(l)}
 &=\mU^{(l+1)}-\mU^{(l)}.
\end{aligned}
\end{equation}
For raw residual states, second differences evaluate changes between consecutive layer updates. In Mamba2, inputs are normalized to compute concatenated head outputs $\vy_t^{\rm all}$ and gate signal $\vz_t$. After gating, $\vw_t=\vy_t^{\rm all}\odot\mathrm{SiLU}(\vz_t)$ undergoes RMSNorm before output projection:
\begin{equation}\label{eq:ssd-full-main}
 \mU^{(l)}_{t,:}=\big[\mathrm{RMSNorm}(\vw_t)\big]^\top\mW_o.
\end{equation}
Here $\odot$ multiplies coordinates, $\mathrm{SiLU}(z)=z/(1+e^{-z})$ acts elementwise, and $\mW_o$ projects back to the hidden dimension. Consider two token vectors $c\ve^{(1)}+\xi_t\ve^{(2)}$ in $\mathbb R^d$, where $c\ve^{(1)}$ ($c>0$) represents a shared context direction and $\xi_t$ encodes token-specific variation. Two value channels represent opposite signs of $\xi_t$, and the output projection subtracts their gated, normalized difference along $\ve^{(2)}$ with update scale $\eta>0$.

\begin{proposition}[Diversity loss through repeated Mamba2 updates]\label{prop:active-main}
A block of the form above, with nonzero state writes and reads, can be constructed such that reusing it across layers with sufficiently small update scale $\eta>0$ drives initially small coordinate differences $\bm{\xi}=(\xi_1,\xi_2)^\top$ toward zero. Consequently, $\mathrm{eRank}(\mH_{\mathrm{raw}}^{(l)})\to1$, while raw second differences remain bounded by $O(\eta^2)$.
\end{proposition}

This construction establishes that repeated Mamba2 updates can compress token diversity while keeping changes between consecutive updates small (Figure~\ref{fig:theory-erank}). This coexistence mechanism provides an architectural interpretation of Findings~\ref{F1} and~\ref{F3} for raw residual states; connecting it quantitatively to empirical terminal rank and $\rho$ additionally requires accounting for final normalization and global depth dispersion.

\subsection{Hybrid Architecture Integration}
To explain how hybrids decouple high answer accessibility from contracting token spectra, we analyze how attention signals survive subsequent recurrent layers. Attention layers in hybrid models like Mamba2Attn inject explicit query-time updates into the recurrent residual stream. To measure propagation, compare two forward passes of the same input, with and without a target attention update. Let $\bm{\Delta}_{\rm in}\ne\bm{0}$ denote their token-matrix difference immediately following the attention layer, and $\bm{\Delta}_{\rm out}$ the difference after $k$ subsequent Mamba2 blocks (before final normalization).

\begin{proposition}[Conditional full-block survival]\label{prop:ssd-survival-main}
If each subsequent complete recurrent branch $F_l$ is $s_l$-Lipschitz in Frobenius norm along the segments connecting both forward paths, with $0\le s_l<1$, then
\begin{equation}\label{eq:ssd-survival-main}
 \prod_l(1-s_l)\le
 \frac{\|\bm{\Delta}_{\rm out}\|_F}{\|\bm{\Delta}_{\rm in}\|_F}
 \le\prod_l(1+s_l),
\end{equation}
where products span the $k$ subsequent recurrent blocks and $\|\cdot\|_F$ denotes the Frobenius norm.
\end{proposition}
Equation~\eqref{eq:ssd-survival-main} proves that attention-induced differences survive subsequent recurrent layers under the stated conditions, even as token diversity contracts. When these surviving differences preserve class separation in normalized probe features, they sustain answer-location decodability without broad token spectra. This explains how hybrids can combine strong task accessibility with low spectral retention (Findings~\ref{F1} and~\ref{F6}).

\finding{
SSM updates can reduce token diversity, while attention enables hybrids to inject query-specific signals that can survive subsequent recurrent layers. Answer locations stay decodable when the resulting probe features are linearly separable.\label{F7}
}

\section{Conclusions}
\label{sec:conclusions}

Our joint analysis reveals three core contrasts: (i) SSMs reach higher peak token diversity than TBMs but retain far less at the output; (ii) reinitialization reverses this output-rank ordering, indicating that the observed compression depends on pretrained parameters; and (iii) TBMs and hybrids preserve central answer locations better than SSMs, with hybrid family means leading final probing on KVPR and approaching TBMs on MDQA. Finally, theoretical analysis links these dynamics to state-rank bottlenecks and attention-update survival.

\section*{Limitations}
\label{sec:limitations}

Our analysis of pretrained TBMs, SSMs and hybrids opens three natural extensions for future work. First, our study focuses on text inputs. Extending the analysis to multimodal models could test whether processing images or audio produces the same patterns of representation diversity, retention and accessibility observed for text. Second, our probes predict answer locations from hidden states. Using intermediate-layer answer-location predictions to select relevant context for generation could test whether this information improves generated-answer accuracy. Third, our pretrained checkpoints differ in architecture and training history. Matched pretraining experiments that hold data and model scale fixed while varying recurrence and attention placement could isolate how these design choices shape representation flow.

\section*{Ethical Considerations}

The study uses publicly available model checkpoints and benchmark datasets. Applying these analyses to private documents or conversations would involve extracting hidden states that encode information from sensitive text; sharing the source text or these states could disclose information about individual inputs. Future implementations should therefore control access to both inputs and extracted states, and investigate ways to report aggregate probing results while protecting individual inputs. In consequential applications, selecting models by answer-location accessibility alone could favor a model that still produces incorrect answers. Future deployment studies should therefore pair probing with generated-answer evaluation and human review, connecting architectural selection to the reliability of the answers users receive.

\bibliography{anthology,custom}

\clearpage
\appendix
\raggedbottom

\section{Implementation Details}
\label{appdx:implementation}

\paragraph{Hybrid architectures.}
Hybrid models vary in attention placement and frequency over their recurrent backbones: Mamba2Attn-2.7B incorporates full causal attention at 6 of 64 blocks ($9.4\%$; layers 10, 19, 28, 37, 46, 57) \citep{mamba2}; Zamba2-1.2B invokes a shared full-attention block at 6 of 38 stages ($15.8\%$) \citep{glorioso2024zamba2suitetechnicalreport}; Qwen3.5-2B places full attention every fourth block (6 of 24, $25.0\%$) interleaved with Gated DeltaNet \citep{qwen2026qwen35}; RecurrentGemma-2B uses RG-LRU recurrence with 2,048-token local attention at 8 of 26 layers ($30.8\%$) \citep{recurrentgemma2024report}; Falcon-H1-1.5B couples parallel attention and Mamba branches across all 24 blocks \citep{falconh1}; and Nemotron-H-8B intersperses four full-attention stages across 52 residual stages \citep{nvidia2025nemotronh}.

\paragraph{Weight reinitialization.}
Reinitialization schemes resample linear projections using Gaussian ($\mathcal{N}(0,0.02^2)$), Xavier-uniform (gain 1), or Kaiming-uniform (fan-in, ReLU) distributions, zeroing linear biases. Resets apply to attention/MLP projections in TBMs, and mixer, recurrent, attention, and MLP projections in SSMs and hybrids. Token embeddings and untied heads are randomized under Gaussian and Xavier schemes for SSMs and hybrids, but kept under Kaiming and across all TBM initializations. Normalization affine parameters, 1D convolution weights, and learned decay/step coefficients are preserved. A fresh-configuration control instantiates untrained 24-layer Transformer and Mamba stacks (width 1024), yielding output normalized ranks of $0.16$--$0.27$ for Transformer versus $0.91$ for Mamba across all initialization schemes.

\paragraph{Probing protocol.}
Following Section~\ref{sssec:probing-analysis}, affine probes $\vf^{(l)}$ are trained on final prompt representations $\vh_{\mathrm{last}}^{(l)}$ across all layers using Adam (150 epochs, learning rate $0.005$, weight decay $0.001$), selecting weights by minimum validation cross-entropy. Probing uses complete task-formatted retrieval prompts rather than the fixed-length inputs used for geometry analyses; prompt lengths depend on the model tokenizer. KVPR evaluates 11 target positions ($1, 5, 10, \dots, 35$) across $2{,}000$ standardized key--value tables ($22{,}000$ inputs; $70/10/20$ train/validation/test group-disjoint splits over complete tables). MDQA probes seven target document positions ($1,3,\ldots,13$) across $18{,}872$ inputs derived from 337 source questions.

% \paragraph{Input populations and split scope.}
% MDQA uses identical inputs across models, retaining only questions whose complete variant--position combinations fit every model's context limit with a 64-token answer reserve. KVPR uses model-specific table populations. The probes therefore measure location classification on held-out tables or questions. Geometry--probe correlations pair geometry computed on $2$K-token inputs with probing profiles over physical Layers~$1,\ldots,L$, relating depth profiles from distinct input sets.

\paragraph{Uncertainty and controls.}
Uncertainty intervals ($95\%$) quantify variation across five random seeds and group splits using Student-$t$ distributions ($df=4$), with peak-layer metrics conditioned on the highest seed-mean validation layer. Permuted-label controls independently verify baseline accuracy by randomly shuffling location labels across dataset inputs.

\paragraph{Prompt templates.}
Following \citet{liu-etal-2024-lost} and \citet{knowbutdonttell}, KVPR places $n=35$ key--value pairs of randomly generated UUIDs in a JSON object and queries the key at target position $k$:
\begin{tcolorbox}[colback=white,colframe=black!75!white,arc=0mm]
\small
Extract the value corresponding to the specified key in the JSON object below.\par\medskip
JSON data:\\
\{\\
"key$^1$": "value$^1$",\\
\ldots\\
"key$^k$": "value$^k$",\\
\ldots\\
"key$^n$": "value$^n$"\\
\}\par\medskip
Key: "key$^k$"\\
Corresponding value:
\end{tcolorbox}
MDQA places the answer-bearing document at target position $k$ among $n=13$ documents, with shuffled distractors elsewhere. The example below illustrates the prompt format; bold marks the answer-bearing document for illustration only.
\begin{tcolorbox}[colback=white,colframe=black!75!white,arc=0mm]
\small
Write a high-quality answer for the given question using only the provided search results (some of which might be irrelevant).\par\medskip
Document [1](Title: Asian Americans in science and technology) Prize in physics for discovery of the subatomic...\par
\ldots\par
\textbf{Document [$k$](Title: List of Nobel laureates in Physics) The first Nobel Prize in Physics was awarded in 1901...}\par
\ldots\par
Document [$n$] (Title: Scientist) and pursued through a unique method, was essentially in place. Ram\'on y Cajal won ...\par\medskip
Question: who got the first nobel prize in physics\\
Answer:
\end{tcolorbox}

\begin{table}[t]
\centering\small
\begin{tabular}{lrr}
\toprule
Model & KVPR & MDQA \\
\midrule
GPT-Neo-1.3B & 9.3 & 14.7 \\
GPT-Neo-2.7B & 9.5 & 14.7 \\
Pythia-1.4B & 9.5 & 14.7 \\
Pythia-2.8B & 9.6 & 14.9 \\
Llama-3.1-8B-Instruct & 9.5 & 14.9 \\
Qwen2.5-32B Base & 9.6 & 14.8 \\
\midrule
Mamba-1.4B & 9.4 & 14.9 \\
Mamba-2.8B & 9.7 & 15.1 \\
Mamba2-1.3B & 9.8 & 14.6 \\
Mamba2-2.7B & 9.6 & 14.8 \\
RWKV-7-1.5B & 9.6 & 14.8 \\
DeltaNet-1.3B & 9.5 & 15.1 \\
Mamba-3 SISO-1.5B & 9.8 & 14.8 \\
Mamba-3 MIMO-1.5B & 9.4 & 14.6 \\
Gated-DeltaNet-1.3B & 9.4 & 14.7 \\
\midrule
Mamba2Attn-2.7B & 9.9 & 14.8 \\
Zamba2-1.2B & 9.6 & 14.8 \\
Qwen3.5-2B & 9.5 & 14.7 \\
RecurrentGemma-2B & 9.8 & 14.7 \\
Falcon-H1-1.5B & 9.4 & 14.7 \\
Nemotron-H-8B & 9.6 & 14.8 \\
\bottomrule
\end{tabular}
\caption{Permuted-label control accuracy (\%). Each entry is the maximum over layers of the five-seed mean, not the control accuracy at the unpermuted probe peak. Chance is $9.1\%$ for KVPR and $14.3\%$ for thirteen-document MDQA.}
\label{tab:probe-controls}
\end{table}

\section{Hybrid Accessibility and Spectra}
\label{appdx:hybrid-accessibility}

\paragraph{Accessibility gains occur at different block types.}
\citet{bae2026hybrid} find that Transformer--Mamba hybrids work best with attention at intermediate depth and relate this placement to representation flow. To examine the associated changes, we compare consecutive-layer gains up to the peak, separating attention-containing stages from recurrent-only stages. On KVPR, Zamba2's hybrid stages gain $15.3$ accuracy points on average, against $2.9$ for recurrent-only stages (one-sided Mann--Whitney $p=0.013$); these stages contain both attention and Mamba2, so the gain belongs to their combination. Qwen3.5 and Mamba2Attn also have their largest single KVPR gains at attention blocks, at Layer~8 ($+18.5$ points) and Layer~19 ($+16.0$ points), respectively, but their mean attention--recurrent differences are not significant ($p=0.47$ and $0.28$). RecurrentGemma instead makes its largest gain at recurrent Layer~7 ($+57.8$ points), showing that high accessibility can be built through different block sequences.

These sequences also differ in their spectral changes. In Zamba2 and Qwen3.5, attention-containing stages show no detected excess rank contraction on KVPR ($p\ge0.37$), while their energy-weighted effective rank, based on $\sigma_k^2/\sum_{k'}\sigma_{k'}^2$, contracts more strongly ($p\le0.003$). The leading--tail decomposition explains how these measures can diverge: concentration in leading directions lowers energy-weighted entropy, while a broad tail can preserve effective rank (Proposition~\ref{prop:tail}, Appendix~\ref{appdx:support}). Attention-associated concentration can thus coexist with spectral diversity, while the accessibility gains above follow architecture-specific block sequences.

\begin{table}[t]
\centering
\footnotesize
\setlength{\tabcolsep}{2.5pt}
\begin{tabular}{lrrrr}
\toprule
 & \multicolumn{2}{c}{KVPR} & \multicolumn{2}{c}{MDQA} \\
\cmidrule(lr){2-3}\cmidrule(lr){4-5}
Model & $S_{\mathrm{rank}}$ & $S_{\mathrm{cos}}$ & $S_{\mathrm{rank}}$ & $S_{\mathrm{cos}}$ \\
\midrule
GPT-Neo-1.3B & $+0.54$ & $-0.14$ & $-0.09$ & $-0.21$ \\
GPT-Neo-2.7B & $+0.56$ & $+0.07$ & $+0.16$ & $+0.07$ \\
Pythia-1.4B & $+0.64$ & $-0.76$ & $+0.62$ & $-0.33$ \\
Pythia-2.8B & $+0.85$ & $-0.82$ & $+0.81$ & $-0.53$ \\
Llama-3.1-8B-Instruct & $+0.31$ & $-0.37$ & $+0.41$ & $+0.27$ \\
Qwen2.5-32B Base & $+0.35$ & $-0.86$ & $+0.83$ & $-0.07$ \\
\midrule
Mamba-1.4B & $-0.32$ & $+0.65$ & $+0.94$ & $+0.02$ \\
Mamba-2.8B & $-0.30$ & $+0.47$ & $+0.49$ & $-0.37$ \\
Mamba2-1.3B & $+0.05$ & $+0.43$ & $+0.86$ & $-0.27$ \\
Mamba2-2.7B & $-0.34$ & $+0.47$ & $+0.95$ & $-0.20$ \\
RWKV-7-1.5B & $+0.74$ & $-0.17$ & $+0.17$ & $-0.71$ \\
DeltaNet-1.3B & $+0.34$ & $-0.81$ & $-0.13$ & $+0.44$ \\
Mamba-3 SISO-1.5B & $+0.62$ & $-0.48$ & $+0.11$ & $+0.76$ \\
Mamba-3 MIMO-1.5B & $+0.45$ & $-0.05$ & $+0.70$ & $-0.58$ \\
Gated-DeltaNet-1.3B & $+0.83$ & $-0.83$ & $+0.12$ & $+0.03$ \\
\midrule
Mamba2Attn-2.7B & $-0.13$ & $-0.22$ & $+0.94$ & $+0.03$ \\
Zamba2-1.2B & $-0.10$ & $+0.11$ & $+0.43$ & $-0.13$ \\
Qwen3.5-2B & $-0.21$ & $-0.39$ & $-0.19$ & $+0.12$ \\
RecurrentGemma-2B & $+0.63$ & $-0.14$ & $+0.51$ & $+0.55$ \\
Falcon-H1-1.5B & $+0.86$ & $-0.01$ & $+0.22$ & $+0.32$ \\
Nemotron-H-8B & $+0.45$ & $-0.04$ & $+0.31$ & $-0.27$ \\
\bottomrule
\end{tabular}
\caption{Within-model, across-layer Spearman correlations of location-probe accuracy with rank ($S_{\mathrm{rank}}$) and InterSim ($S_{\mathrm{cos}}$). Geometry and probing use separate inputs.}
\label{tab:cosine-correlations}
\end{table}

\paragraph{Rank--accessibility coupling.}
Across all twenty-one models in Table~\ref{tab:probing-regen}, the within-model Spearman correlations $S_{\mathrm{rank}}$ depend on both family and task. On KVPR, all six TBMs have positive correlations ($+0.31$--$+0.85$), six of nine SSMs are positive and three negative ($-0.34$--$+0.83$), and the six hybrids split equally between positive and negative correlations ($-0.21$--$+0.86$). On MDQA, five of six TBMs are positive ($-0.09$--$+0.83$), as are eight of nine SSMs ($-0.13$--$+0.95$) and five of six hybrids ($-0.19$--$+0.94$). These signs describe individual depth profiles, whereas family-mean probing accuracies use a twenty-model cohort excluding Qwen2.5-32B.

Directional alignment gives a different pattern in the same twenty-one-model census (Table~\ref{tab:cosine-correlations}). On KVPR, $S_{\mathrm{cos}}$ is positive for one of six TBMs ($-0.86$--$+0.07$), four of nine SSMs ($-0.83$--$+0.65$), and one of six hybrids ($-0.39$--$+0.11$); the remaining correlations are negative. On MDQA, two of six TBMs ($-0.53$--$+0.27$), four of nine SSMs ($-0.71$--$+0.76$), and four of six hybrids ($-0.27$--$+0.55$) are positive. Together, these signs show task- and model-dependent associations between location accessibility, spectral rank and directional alignment. These associations pair geometry and probe depth profiles from distinct input populations (Appendix~\ref{appdx:implementation}).

\begin{table}[t]
\centering
\footnotesize
\setlength{\tabcolsep}{3pt}
\begin{tabular}{llc}
\toprule
Task & Model & Rank ret. (\%) \\
\midrule
KVPR & Mamba2-2.7B & 10.0 \\
 & Mamba2Attn-2.7B & 11.7 \\
MDQA & Mamba2-2.7B & 50.2 \\
 & Mamba2Attn-2.7B & 42.8 \\
\bottomrule
\end{tabular}
\caption{Task-specific Mamba2 rank retention (\%) for Table~\ref{tab:probing-regen}'s probe comparison. Geometry and probing use separate inputs.}
\label{tab:mamba2-task-geometry}
\end{table}

\section{Dataset-wise Representation Flow}
\label{appdx:extra-datasets}

The seven datasets span controlled retrieval and natural-language inputs from distinct domains: KVPR, MDQA, AG News, SAMSum \citep{gliwa-etal-2019-samsum}, CNN/DailyMail, Qasper \citep{dasigi-etal-2021-dataset} and GovReport \citep{huang-etal-2021-efficient}. Figure~\ref{fig:extra-geometry} places their representation-flow profiles side by side, revealing which family contrasts hold across inputs and how their magnitude varies. Table~\ref{tab:dataset-geometry} summarizes token-geometry family means; Tables~\ref{tab:extra-ag-news}--\ref{tab:extra-govreport} give model-wise token geometry and native-depth local and global variation for the five non-retrieval corpora.

\paragraph{Directional alignment.}
For the InterSim profiles in Figure~\ref{fig:rank-vs-intersim}(a), let $\tilde\vh_i^{(l)}$ be the mean representation at token position $i$ over test inputs. Mean pairwise cosine similarity between distinct positions is
\begin{equation}
\mathrm{InterSim}^{(l)}=\frac{2}{n(n-1)}\sum_{i<j}
\frac{\tilde\vh_i^{(l)}\cdot\tilde\vh_j^{(l)}}{\|\tilde\vh_i^{(l)}\|\,\|\tilde\vh_j^{(l)}\|}.
\end{equation}
Higher values indicate stronger directional alignment. Table~\ref{tab:intersim-retention} summarizes final-layer InterSim relative to its maximum over layers for the seven-dataset comparison in Table~\ref{tab:adjudication}.

\begin{table}[t]
\centering
\footnotesize
\setlength{\tabcolsep}{2pt}
\begin{tabular}{lrr}
\toprule
Model & \shortstack{Rank retention\\(\%)} & \shortstack{InterSim final/max\\(\%)} \\
\midrule
GPT-Neo-1.3B & 96.5 & 77.6 \\
GPT-Neo-2.7B & 90.1 & 74.5 \\
Pythia-1.4B & 93.3 & 75.8 \\
Pythia-2.8B & 91.4 & 72.7 \\
Llama-3.1-8B-Instruct & 89.5 & 86.9 \\
Qwen2.5-32B Base & 63.2 & 87.3 \\
\textit{Mean} & \textit{87.3} & \textit{79.1} \\
\midrule
Mamba-1.4B & 35.4 & 90.9 \\
Mamba-2.8B & 30.1 & 100.0 \\
Mamba2-1.3B & 35.3 & 83.4 \\
Mamba2-2.7B & 36.0 & 73.2 \\
RWKV-7-1.5B & 62.8 & 100.0 \\
DeltaNet-1.3B & 93.1 & 66.8 \\
Mamba-3 SISO-1.5B & 48.4 & 94.1 \\
Mamba-3 MIMO-1.5B & 63.7 & 94.0 \\
Gated-DeltaNet-1.3B & 92.2 & 81.4 \\
\textit{Mean} & \textit{55.2} & \textit{87.1} \\
\midrule
Mamba2Attn-2.7B & 28.9 & 99.0 \\
Zamba2-1.2B & 70.4 & 97.7 \\
Qwen3.5-2B & 82.9 & 90.9 \\
RecurrentGemma-2B & 65.3 & 80.7 \\
Falcon-H1-1.5B & 15.2 & 100.0 \\
Nemotron-H-8B & 73.8 & 100.0 \\
\textit{Mean} & \textit{56.1} & \textit{94.7} \\
\bottomrule
\end{tabular}
\caption{Rank retention and final-to-maximum InterSim (\%), computed within each dataset, then averaged equally over seven datasets. \emph{Mean} rows average models within families.}
\label{tab:intersim-retention}
\end{table}

\begin{table*}[t]
\centering
\footnotesize
\setlength{\tabcolsep}{2.4pt}
\begin{tabular}{l ccc ccc ccc}
\toprule
 & \multicolumn{3}{c}{TBM} & \multicolumn{3}{c}{SSM} & \multicolumn{3}{c}{Hybrid} \\
\cmidrule(lr){2-4}\cmidrule(lr){5-7}\cmidrule(lr){8-10}
Dataset & Peak & Ret. & Gau. & Peak & Ret. & Gau. & Peak & Ret. & Gau. \\
\midrule
KVPR & 0.332 & 61.8 & 0.146 & 0.374 & 24.1 & 0.707 & 0.351 & 33.6 & 0.680 \\
MDQA & 0.568 & 94.2 & 0.273 & 0.562 & 67.0 & 0.737 & 0.551 & 66.1 & 0.705 \\
AG News & 0.436 & 93.0 & 0.235 & 0.444 & 62.0 & 0.630 & 0.454 & 58.9 & 0.655 \\
SAMSum & 0.477 & 88.2 & 0.232 & 0.487 & 55.8 & 0.703 & 0.506 & 55.0 & 0.690 \\
CNN/DailyMail & 0.546 & 92.8 & 0.267 & 0.561 & 62.6 & 0.732 & 0.558 & 62.5 & 0.707 \\
Qasper & 0.484 & 89.4 & 0.238 & 0.506 & 56.5 & 0.703 & 0.523 & 56.2 & 0.689 \\
GovReport & 0.488 & 92.1 & 0.240 & 0.518 & 58.5 & 0.725 & 0.520 & 60.1 & 0.701 \\
\bottomrule
\end{tabular}
\caption{Dataset-wise family means over six TBMs, nine SSMs and six hybrids: pretrained peak normalized rank, rank retention (\% of peak), and Gaussian-reinitialized final rank.}
\label{tab:dataset-geometry}
\end{table*}

Across all seven datasets, mean rank retention is higher in TBMs ($61.8$--$94.2\%$) than in SSMs ($24.1$--$67.0\%$) and hybrids ($33.6$--$66.1\%$). Hybrid means lie between the pure families on KVPR and GovReport, but below SSM means on MDQA, AG News, SAMSum, CNN/DailyMail and Qasper. Averaging within datasets and then across models gives seven-dataset family means of $87.3\%$, $55.2\%$ and $56.1\%$ for TBMs, SSMs and hybrids, respectively (Table~\ref{tab:adjudication}); individual models' seven-dataset means span $63.2$--$96.5\%$, $30.1$--$93.1\%$ and $15.2$--$82.9\%$. All three families retain less diversity on KVPR than on the other six datasets, while DeltaNet retains $100\%$ on every dataset except KVPR.

The reinitialization comparison separates this input dependence from the role of pretrained weights. Averaged over seven datasets, TBMs have higher pretrained final rank than under each reinitialization scheme, whereas SSMs and hybrids have lower family means (Table~\ref{tab:adjudication}). Individual departures include Qwen2.5-32B under Gaussian reinitialization, RWKV-7 under Kaiming on AG News ($0.311$ pretrained versus $0.308$ reinitialized) and MDQA ($0.449$ versus $0.420$), and Nemotron-H under Kaiming on six datasets. The five non-retrieval tables retain both the Qwen2.5 Gaussian reversal and the Nemotron-H Kaiming reversal on every corpus; for example, their AG News pretrained/reinitialized final ranks are $0.363/0.585$ and $0.349/0.320$, respectively. These model-specific responses accompany the opposite family-average contrasts in Finding~\ref{F2}.

\paragraph{Sharpness retention.}
Following \citet{wang2025understanding}, sharpness measures normalized pairwise squared distances, $E(\mH)=\sum_{i\ne j}\|\vh_i-\vh_j\|^2/[2(n-1)\sum_i\|\vh_i\|^2]$. For each model and dataset, let $\bar E^{(l)}$ be its layer profile averaged over inputs; retention is $100\bar E^{(L)}/\max_{l\ge1}\bar E^{(l)}$ percent. Family means average models.

Peak rank also has an input-dependent family ordering: SSM means ($0.374$--$0.562$) exceed TBM means ($0.332$--$0.568$) on six datasets, with MDQA reversing the comparison ($0.568$ TBM versus $0.562$ SSM). Hybrid means range from $0.351$ to $0.558$ and are highest on AG News, SAMSum, Qasper and GovReport; SSMs lead on KVPR and CNN/DailyMail, and TBMs on MDQA. The corresponding seven-dataset family means are $0.476$, $0.493$ and $0.495$ for TBMs, SSMs and hybrids (Table~\ref{tab:adjudication}), while individual model means span $0.329$--$0.576$, $0.371$--$0.606$ and $0.448$--$0.536$.

These peak and retention differences need not appear in directional alignment. For Mamba2-1.3B on AG News, rank falls from $0.423$ at Layer~19 to $0.170$ at the output while InterSim changes only from $0.947$ to $0.911$. On Qasper, both change substantially, from $0.507$ to $0.167$ for rank and $0.731$ to $0.527$ for InterSim. Sharpness further distinguishes the five non-retrieval corpora: hybrid mean sharpness retention is $46.1$--$51.5\%$, below SSM means of $61.2$--$74.9\%$ on every corpus, while hybrid rank retention exceeds SSMs only on GovReport. These comparisons separate family-average contrasts from input-dependent compression and individual hybrid behavior. Positional accessibility remains a distinct question, which we test only on KVPR and MDQA because the other inputs lack controlled target positions.

\begin{figure*}[p]
\centering
\includegraphics[width=\textwidth]{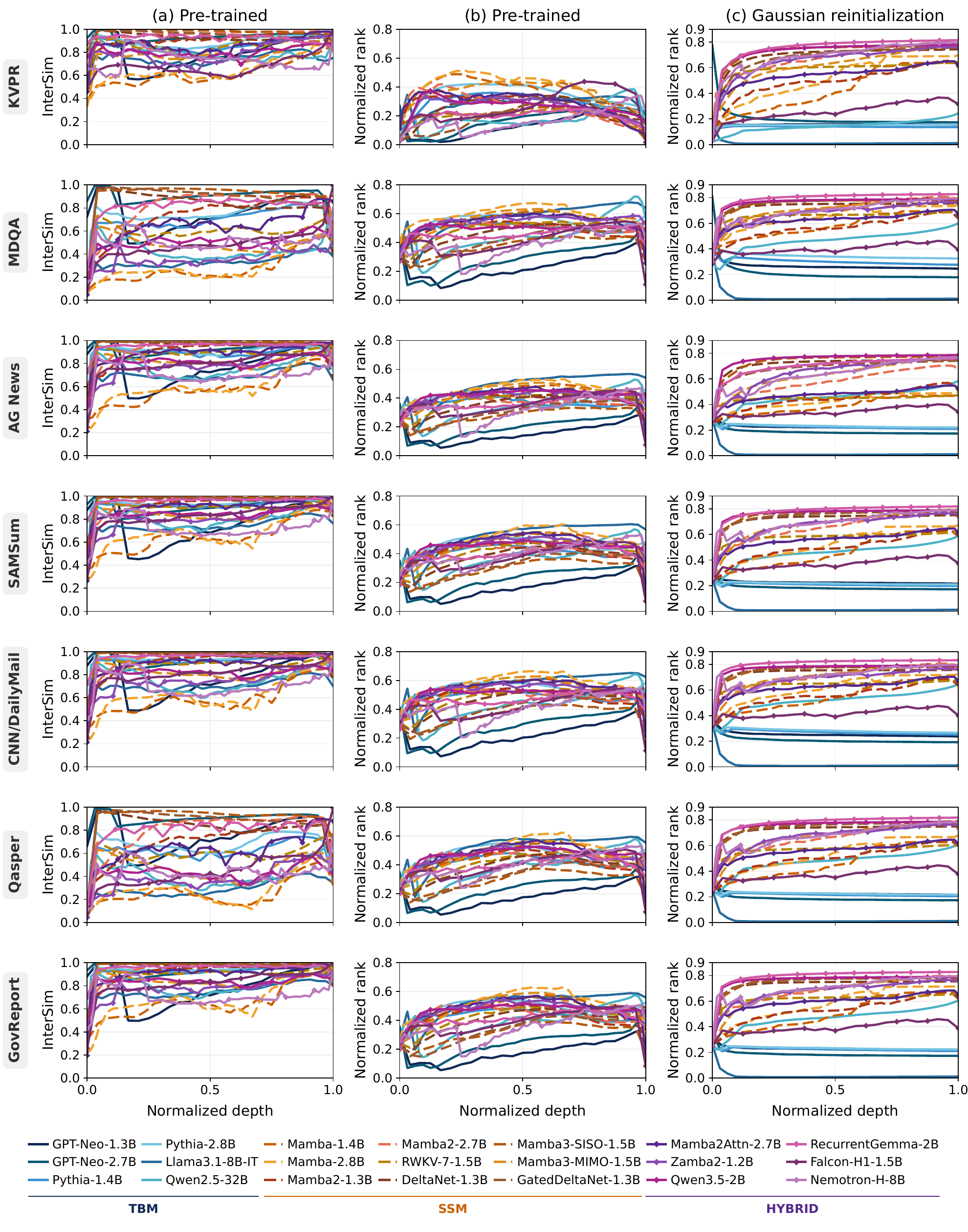}
\caption{Dataset-wise representation flow on seven input sets (rows): pretrained InterSim, pretrained normalized effective rank, and Gaussian-reinitialized rank (columns). All twenty-one models appear in every row, with shared column-wise axes and the colors and line styles of Figure~\ref{fig:rank-vs-intersim}.}
\label{fig:extra-geometry}
\end{figure*}

\begin{table*}[p]
\centering
\footnotesize
\setlength{\tabcolsep}{4pt}
\begin{tabular}{l cccc cccc cc}
\toprule
 & \multicolumn{2}{c}{Rank} & $E(\mH)$ & InterSim & \multicolumn{4}{c}{Final rank} & \multicolumn{2}{c}{Variation} \\
\cmidrule(lr){2-3}\cmidrule(lr){4-4}\cmidrule(lr){5-5}\cmidrule(lr){6-9}\cmidrule(lr){10-11}
Model & Peak & Ret. & Ret. & F/M & Pre. & Gau. & Xav. & Kai. & $V_{\mathrm{loc}}$ & $V_{\mathrm{glob}}$ \\
\midrule
GPT-Neo-1.3B & .294 & 100.0 & 100.0 & 84.0 & .294 & .213 & .087 & .097 & 0.97 & 2.98 \\
GPT-Neo-2.7B & .341 & 94.2 & 100.0 & 77.5 & .321 & .173 & .084 & .094 & 0.84 & 3.12 \\
Pythia-1.4B & .424 & 100.0 & 100.0 & 80.2 & .424 & .210 & .121 & .166 & 0.30 & 1.05 \\
Pythia-2.8B & .459 & 100.0 & 100.0 & 76.1 & .459 & .221 & .123 & .162 & 0.28 & 1.21 \\
Llama-3.1-8B-Instruct & .568 & 95.3 & 81.9 & 81.9 & .541 & .012 & .007 & .010 & 0.07 & 0.36 \\
Qwen2.5-32B Base & .528 & 68.7 & 63.8 & 90.1 & .363 & .585 & .009 & .077 & 0.73 & 3.53 \\
\textit{Mean} & \textit{.436} & \textit{93.0} & \textit{90.9} & \textit{81.7} & \textit{.400} & \textit{.235} & \textit{.072} & \textit{.101} & \textit{0.53} & \textit{2.04} \\
\midrule
Mamba-1.4B & .497 & 40.1 & 71.7 & 96.2 & .199 & .475 & .432 & .532 & 0.11 & 0.71 \\
Mamba-2.8B & .531 & 33.8 & 33.6 & 100.0 & .179 & .498 & .444 & .519 & 0.11 & 0.83 \\
Mamba2-1.3B & .423 & 40.2 & 56.6 & 92.9 & .170 & .527 & .435 & .611 & 1.31 & 7.60 \\
Mamba2-2.7B & .424 & 40.7 & 72.5 & 85.8 & .173 & .662 & .564 & .712 & 1.32 & 8.94 \\
RWKV-7-1.5B & .433 & 71.7 & 54.8 & 100.0 & .311 & .472 & .472 & .308 & 0.62 & 2.46 \\
DeltaNet-1.3B & .436 & 100.0 & 100.0 & 66.1 & .436 & .750 & .697 & .748 & 1.08 & 4.26 \\
Mamba-3 SISO-1.5B & .333 & 59.3 & 96.9 & 97.6 & .197 & .767 & .778 & .640 & 2.39 & 9.85 \\
Mamba-3 MIMO-1.5B & .507 & 72.3 & 87.6 & 92.2 & .367 & .776 & .798 & .672 & 0.49 & 2.01 \\
Gated-DeltaNet-1.3B & .408 & 100.0 & 100.0 & 85.8 & .408 & .741 & .665 & .687 & 0.86 & 3.27 \\
\textit{Mean} & \textit{.444} & \textit{62.0} & \textit{74.9} & \textit{90.7} & \textit{.271} & \textit{.630} & \textit{.587} & \textit{.603} & \textit{0.92} & \textit{4.44} \\
\midrule
Mamba2Attn-2.7B & .477 & 30.2 & 23.4 & 99.8 & .144 & .522 & .424 & .584 & 0.96 & 6.89 \\
Zamba2-1.2B & .451 & 76.0 & 60.1 & 99.6 & .343 & .753 & .701 & .774 & 0.13 & 0.78 \\
Qwen3.5-2B & .455 & 90.0 & 94.7 & 88.6 & .409 & .786 & .794 & .777 & 0.09 & 0.51 \\
RecurrentGemma-2B & .413 & 66.7 & 100.0 & 86.3 & .275 & .760 & .828 & .770 & 1.84 & 5.72 \\
Falcon-H1-1.5B & .454 & 16.5 & 1.4 & 100.0 & .075 & .342 & .321 & .369 & 0.93 & 4.75 \\
Nemotron-H-8B & .470 & 74.2 & 29.2 & 100.0 & .349 & .767 & .562 & .320 & 0.25 & 1.28 \\
\textit{Mean} & \textit{.454} & \textit{58.9} & \textit{51.5} & \textit{95.7} & \textit{.266} & \textit{.655} & \textit{.605} & \textit{.599} & \textit{0.70} & \textit{3.32} \\
\bottomrule
\end{tabular}
\caption{AG News geometry: peak rank, rank and sharpness retention (\%), InterSim final-to-maximum ratio (\%), final rank before and after reinitialization, and native-depth local and global variation.}
\label{tab:extra-ag-news}
\end{table*}

\begin{table*}[p]
\centering
\footnotesize
\setlength{\tabcolsep}{4pt}
\begin{tabular}{l cccc cccc cc}
\toprule
 & \multicolumn{2}{c}{Rank} & $E(\mH)$ & InterSim & \multicolumn{4}{c}{Final rank} & \multicolumn{2}{c}{Variation} \\
\cmidrule(lr){2-3}\cmidrule(lr){4-4}\cmidrule(lr){5-5}\cmidrule(lr){6-9}\cmidrule(lr){10-11}
Model & Peak & Ret. & Ret. & F/M & Pre. & Gau. & Xav. & Kai. & $V_{\mathrm{loc}}$ & $V_{\mathrm{glob}}$ \\
\midrule
GPT-Neo-1.3B & .315 & 98.9 & 100.0 & 88.1 & .311 & .215 & .086 & .096 & 0.98 & 3.02 \\
GPT-Neo-2.7B & .389 & 86.4 & 100.0 & 85.5 & .336 & .172 & .084 & .092 & 0.85 & 3.23 \\
Pythia-1.4B & .464 & 94.7 & 88.0 & 89.6 & .439 & .198 & .113 & .158 & 0.28 & 1.01 \\
Pythia-2.8B & .521 & 92.3 & 100.0 & 87.5 & .481 & .209 & .116 & .156 & 0.27 & 1.15 \\
Llama-3.1-8B-Instruct & .605 & 93.8 & 74.1 & 89.8 & .567 & .012 & .007 & .010 & 0.07 & 0.35 \\
Qwen2.5-32B Base & .570 & 63.5 & 56.2 & 93.3 & .362 & .588 & .009 & .079 & 0.71 & 3.49 \\
\textit{Mean} & \textit{.477} & \textit{88.2} & \textit{86.4} & \textit{89.0} & \textit{.416} & \textit{.232} & \textit{.069} & \textit{.099} & \textit{0.53} & \textit{2.04} \\
\midrule
Mamba-1.4B & .550 & 34.0 & 61.4 & 100.0 & .187 & .619 & .551 & .673 & 0.12 & 0.76 \\
Mamba-2.8B & .603 & 28.4 & 22.9 & 100.0 & .171 & .673 & .581 & .704 & 0.11 & 0.88 \\
Mamba2-1.3B & .477 & 32.8 & 65.3 & 90.9 & .156 & .587 & .482 & .665 & 1.36 & 7.95 \\
Mamba2-2.7B & .470 & 33.9 & 79.8 & 85.6 & .159 & .730 & .639 & .765 & 1.39 & 9.42 \\
RWKV-7-1.5B & .491 & 65.4 & 52.6 & 100.0 & .321 & .616 & .587 & .353 & 0.68 & 2.73 \\
DeltaNet-1.3B & .468 & 100.0 & 100.0 & 82.4 & .468 & .751 & .699 & .749 & 1.00 & 4.22 \\
Mamba-3 SISO-1.5B & .363 & 45.1 & 56.7 & 98.8 & .164 & .773 & .782 & .653 & 2.46 & 10.29 \\
Mamba-3 MIMO-1.5B & .550 & 63.0 & 68.5 & 98.7 & .346 & .795 & .804 & .690 & 0.50 & 2.12 \\
Gated-DeltaNet-1.3B & .407 & 100.0 & 100.0 & 92.6 & .407 & .780 & .716 & .746 & 0.86 & 3.36 \\
\textit{Mean} & \textit{.487} & \textit{55.8} & \textit{67.5} & \textit{94.3} & \textit{.264} & \textit{.703} & \textit{.649} & \textit{.666} & \textit{0.94} & \textit{4.64} \\
\midrule
Mamba2Attn-2.7B & .548 & 24.9 & 17.4 & 100.0 & .137 & .615 & .484 & .681 & 0.94 & 6.83 \\
Zamba2-1.2B & .517 & 68.4 & 52.9 & 100.0 & .354 & .756 & .717 & .774 & 0.13 & 0.79 \\
Qwen3.5-2B & .491 & 85.5 & 83.5 & 95.4 & .419 & .788 & .795 & .781 & 0.09 & 0.49 \\
RecurrentGemma-2B & .466 & 62.9 & 100.0 & 91.2 & .293 & .819 & .838 & .802 & 1.75 & 5.49 \\
Falcon-H1-1.5B & .481 & 14.6 & 1.1 & 100.0 & .070 & .375 & .355 & .406 & 0.90 & 4.71 \\
Nemotron-H-8B & .531 & 73.7 & 26.6 & 100.0 & .391 & .787 & .591 & .345 & 0.26 & 1.35 \\
\textit{Mean} & \textit{.506} & \textit{55.0} & \textit{46.9} & \textit{97.8} & \textit{.277} & \textit{.690} & \textit{.630} & \textit{.632} & \textit{0.68} & \textit{3.28} \\
\bottomrule
\end{tabular}
\caption{SAMSum geometry: peak rank, rank and sharpness retention (\%), InterSim final-to-maximum ratio (\%), final rank before and after reinitialization, and native-depth local and global variation.}
\label{tab:extra-samsum}
\end{table*}

\begin{table*}[p]
\centering
\footnotesize
\setlength{\tabcolsep}{4pt}
\begin{tabular}{l cccc cccc cc}
\toprule
 & \multicolumn{2}{c}{Rank} & $E(\mH)$ & InterSim & \multicolumn{4}{c}{Final rank} & \multicolumn{2}{c}{Variation} \\
\cmidrule(lr){2-3}\cmidrule(lr){4-4}\cmidrule(lr){5-5}\cmidrule(lr){6-9}\cmidrule(lr){10-11}
Model & Peak & Ret. & Ret. & F/M & Pre. & Gau. & Xav. & Kai. & $V_{\mathrm{loc}}$ & $V_{\mathrm{glob}}$ \\
\midrule
GPT-Neo-1.3B & .397 & 100.0 & 100.0 & 85.3 & .397 & .237 & .099 & .111 & 0.99 & 3.08 \\
GPT-Neo-2.7B & .459 & 93.1 & 100.0 & 80.9 & .427 & .192 & .092 & .099 & 0.86 & 3.29 \\
Pythia-1.4B & .535 & 100.0 & 100.0 & 81.1 & .535 & .252 & .148 & .189 & 0.26 & 0.91 \\
Pythia-2.8B & .589 & 99.1 & 100.0 & 78.0 & .584 & .267 & .150 & .181 & 0.26 & 1.07 \\
Llama-3.1-8B-Instruct & .652 & 95.7 & 77.9 & 85.8 & .624 & .012 & .007 & .009 & 0.07 & 0.36 \\
Qwen2.5-32B Base & .643 & 69.0 & 55.6 & 93.7 & .444 & .645 & .010 & .090 & 0.75 & 3.75 \\
\textit{Mean} & \textit{.546} & \textit{92.8} & \textit{88.9} & \textit{84.1} & \textit{.502} & \textit{.267} & \textit{.084} & \textit{.113} & \textit{0.53} & \textit{2.08} \\
\midrule
Mamba-1.4B & .611 & 43.1 & 67.1 & 97.3 & .263 & .674 & .609 & .710 & 0.13 & 0.83 \\
Mamba-2.8B & .665 & 38.0 & 23.5 & 100.0 & .253 & .729 & .653 & .753 & 0.12 & 0.94 \\
Mamba2-1.3B & .544 & 44.1 & 67.4 & 90.0 & .240 & .634 & .533 & .692 & 1.43 & 8.52 \\
Mamba2-2.7B & .552 & 44.1 & 87.1 & 79.7 & .244 & .756 & .689 & .780 & 1.44 & 10.02 \\
RWKV-7-1.5B & .541 & 74.0 & 55.6 & 100.0 & .400 & .674 & .668 & .415 & 0.71 & 2.82 \\
DeltaNet-1.3B & .564 & 100.0 & 100.0 & 75.4 & .564 & .758 & .709 & .755 & 1.09 & 4.49 \\
Mamba-3 SISO-1.5B & .430 & 50.7 & 31.5 & 99.6 & .218 & .773 & .783 & .644 & 2.53 & 10.54 \\
Mamba-3 MIMO-1.5B & .627 & 69.8 & 57.0 & 100.0 & .438 & .799 & .805 & .697 & 0.53 & 2.19 \\
Gated-DeltaNet-1.3B & .510 & 100.0 & 100.0 & 92.4 & .510 & .787 & .727 & .757 & 0.89 & 3.49 \\
\textit{Mean} & \textit{.561} & \textit{62.6} & \textit{65.5} & \textit{92.7} & \textit{.348} & \textit{.732} & \textit{.686} & \textit{.689} & \textit{0.99} & \textit{4.87} \\
\midrule
Mamba2Attn-2.7B & .606 & 35.7 & 18.7 & 100.0 & .216 & .670 & .549 & .725 & 1.05 & 7.63 \\
Zamba2-1.2B & .579 & 82.0 & 60.8 & 100.0 & .475 & .757 & .721 & .774 & 0.14 & 0.88 \\
Qwen3.5-2B & .537 & 92.3 & 87.3 & 94.6 & .495 & .790 & .796 & .783 & 0.09 & 0.52 \\
RecurrentGemma-2B & .534 & 71.7 & 100.0 & 89.3 & .383 & .833 & .841 & .819 & 1.80 & 5.79 \\
Falcon-H1-1.5B & .545 & 20.7 & 1.3 & 100.0 & .113 & .401 & .388 & .423 & 0.93 & 4.85 \\
Nemotron-H-8B & .544 & 72.8 & 27.8 & 100.0 & .396 & .793 & .622 & .350 & 0.25 & 1.28 \\
\textit{Mean} & \textit{.558} & \textit{62.5} & \textit{49.3} & \textit{97.3} & \textit{.347} & \textit{.707} & \textit{.653} & \textit{.646} & \textit{0.71} & \textit{3.49} \\
\bottomrule
\end{tabular}
\caption{CNN/DailyMail geometry: peak rank, rank and sharpness retention (\%), InterSim final-to-maximum ratio (\%), final rank before and after reinitialization, and native-depth local and global variation.}
\label{tab:extra-cnn-dailymail}
\end{table*}

\begin{table*}[p]
\centering
\footnotesize
\setlength{\tabcolsep}{4pt}
\begin{tabular}{l cccc cccc cc}
\toprule
 & \multicolumn{2}{c}{Rank} & $E(\mH)$ & InterSim & \multicolumn{4}{c}{Final rank} & \multicolumn{2}{c}{Variation} \\
\cmidrule(lr){2-3}\cmidrule(lr){4-4}\cmidrule(lr){5-5}\cmidrule(lr){6-9}\cmidrule(lr){10-11}
Model & Peak & Ret. & Ret. & F/M & Pre. & Gau. & Xav. & Kai. & $V_{\mathrm{loc}}$ & $V_{\mathrm{glob}}$ \\
\midrule
GPT-Neo-1.3B & .320 & 100.0 & 100.0 & 48.6 & .320 & .215 & .088 & .101 & 0.97 & 2.95 \\
GPT-Neo-2.7B & .386 & 94.1 & 100.0 & 37.8 & .363 & .173 & .085 & .095 & 0.83 & 3.14 \\
Pythia-1.4B & .480 & 95.1 & 100.0 & 46.0 & .457 & .205 & .120 & .165 & 0.28 & 1.00 \\
Pythia-2.8B & .548 & 92.5 & 100.0 & 42.1 & .506 & .217 & .123 & .160 & 0.27 & 1.17 \\
Llama-3.1-8B-Instruct & .595 & 93.0 & 69.8 & 80.5 & .553 & .012 & .007 & .009 & 0.07 & 0.35 \\
Qwen2.5-32B Base & .578 & 61.5 & 52.4 & 74.6 & .355 & .609 & .009 & .080 & 0.70 & 3.44 \\
\textit{Mean} & \textit{.484} & \textit{89.4} & \textit{87.0} & \textit{54.9} & \textit{.426} & \textit{.238} & \textit{.072} & \textit{.102} & \textit{0.52} & \textit{2.01} \\
\midrule
Mamba-1.4B & .569 & 35.8 & 76.9 & 61.8 & .203 & .625 & .556 & .676 & 0.12 & 0.77 \\
Mamba-2.8B & .623 & 29.0 & 26.4 & 100.0 & .180 & .680 & .583 & .710 & 0.12 & 0.91 \\
Mamba2-1.3B & .507 & 33.0 & 65.2 & 59.4 & .167 & .595 & .489 & .672 & 1.36 & 8.08 \\
Mamba2-2.7B & .493 & 34.5 & 84.1 & 38.4 & .170 & .727 & .644 & .766 & 1.40 & 9.58 \\
RWKV-7-1.5B & .491 & 64.7 & 57.6 & 100.0 & .318 & .606 & .586 & .353 & 0.71 & 2.85 \\
DeltaNet-1.3B & .488 & 100.0 & 100.0 & 40.2 & .488 & .751 & .700 & .749 & 1.01 & 4.29 \\
Mamba-3 SISO-1.5B & .374 & 48.8 & 80.0 & 74.2 & .183 & .771 & .781 & .647 & 2.40 & 10.20 \\
Mamba-3 MIMO-1.5B & .570 & 63.0 & 71.6 & 80.4 & .359 & .791 & .803 & .691 & 0.52 & 2.24 \\
Gated-DeltaNet-1.3B & .436 & 100.0 & 100.0 & 56.5 & .436 & .780 & .717 & .745 & 0.88 & 3.49 \\
\textit{Mean} & \textit{.506} & \textit{56.5} & \textit{73.5} & \textit{67.9} & \textit{.278} & \textit{.703} & \textit{.651} & \textit{.668} & \textit{0.95} & \textit{4.71} \\
\midrule
Mamba2Attn-2.7B & .569 & 27.1 & 20.1 & 97.7 & .154 & .616 & .489 & .681 & 0.91 & 6.58 \\
Zamba2-1.2B & .532 & 68.7 & 60.7 & 84.5 & .366 & .755 & .719 & .774 & 0.14 & 0.84 \\
Qwen3.5-2B & .526 & 86.9 & 86.5 & 74.7 & .457 & .786 & .795 & .780 & 0.09 & 0.51 \\
RecurrentGemma-2B & .473 & 68.3 & 100.0 & 54.6 & .323 & .819 & .840 & .811 & 1.69 & 5.19 \\
Falcon-H1-1.5B & .506 & 14.3 & 1.0 & 100.0 & .073 & .377 & .355 & .398 & 0.89 & 4.69 \\
Nemotron-H-8B & .532 & 71.9 & 29.4 & 100.0 & .382 & .781 & .598 & .338 & 0.24 & 1.22 \\
\textit{Mean} & \textit{.523} & \textit{56.2} & \textit{49.6} & \textit{85.2} & \textit{.292} & \textit{.689} & \textit{.633} & \textit{.630} & \textit{0.66} & \textit{3.17} \\
\bottomrule
\end{tabular}
\caption{Qasper geometry: peak rank, rank and sharpness retention (\%), InterSim final-to-maximum ratio (\%), final rank before and after reinitialization, and native-depth local and global variation.}
\label{tab:extra-qasper}
\end{table*}

\begin{table*}[tp]
\centering
\footnotesize
\setlength{\tabcolsep}{4pt}
\begin{tabular}{l cccc cccc cc}
\toprule
 & \multicolumn{2}{c}{Rank} & $E(\mH)$ & InterSim & \multicolumn{4}{c}{Final rank} & \multicolumn{2}{c}{Variation} \\
\cmidrule(lr){2-3}\cmidrule(lr){4-4}\cmidrule(lr){5-5}\cmidrule(lr){6-9}\cmidrule(lr){10-11}
Model & Peak & Ret. & Ret. & F/M & Pre. & Gau. & Xav. & Kai. & $V_{\mathrm{loc}}$ & $V_{\mathrm{glob}}$ \\
\midrule
GPT-Neo-1.3B & .341 & 100.0 & 86.3 & 91.4 & .341 & .218 & .089 & .104 & 0.96 & 2.99 \\
GPT-Neo-2.7B & .392 & 96.5 & 100.0 & 87.6 & .378 & .173 & .083 & .092 & 0.84 & 3.22 \\
Pythia-1.4B & .482 & 98.8 & 90.2 & 86.9 & .477 & .211 & .123 & .167 & 0.26 & 0.93 \\
Pythia-2.8B & .548 & 96.4 & 100.0 & 84.1 & .528 & .223 & .123 & .155 & 0.26 & 1.09 \\
Llama-3.1-8B-Instruct & .589 & 94.9 & 64.1 & 94.4 & .559 & .012 & .007 & .009 & 0.07 & 0.36 \\
Qwen2.5-32B Base & .575 & 65.8 & 43.7 & 97.0 & .378 & .606 & .009 & .079 & 0.72 & 3.68 \\
\textit{Mean} & \textit{.488} & \textit{92.1} & \textit{80.7} & \textit{90.2} & \textit{.444} & \textit{.240} & \textit{.072} & \textit{.101} & \textit{0.52} & \textit{2.05} \\
\midrule
Mamba-1.4B & .573 & 38.9 & 63.1 & 95.0 & .223 & .663 & .594 & .704 & 0.13 & 0.84 \\
Mamba-2.8B & .630 & 31.6 & 16.2 & 100.0 & .199 & .726 & .632 & .747 & 0.13 & 0.99 \\
Mamba2-1.3B & .513 & 36.6 & 58.5 & 93.1 & .188 & .619 & .509 & .687 & 1.40 & 8.47 \\
Mamba2-2.7B & .504 & 38.8 & 80.1 & 85.4 & .195 & .748 & .675 & .776 & 1.42 & 9.97 \\
RWKV-7-1.5B & .502 & 68.1 & 53.7 & 100.0 & .342 & .656 & .632 & .379 & 0.72 & 2.91 \\
DeltaNet-1.3B & .506 & 100.0 & 100.0 & 79.8 & .506 & .755 & .706 & .753 & 1.04 & 4.37 \\
Mamba-3 SISO-1.5B & .390 & 46.7 & 27.1 & 99.8 & .182 & .772 & .782 & .631 & 2.39 & 10.14 \\
Mamba-3 MIMO-1.5B & .592 & 66.0 & 52.2 & 100.0 & .390 & .795 & .804 & .672 & 0.53 & 2.23 \\
Gated-DeltaNet-1.3B & .451 & 100.0 & 100.0 & 92.5 & .451 & .786 & .725 & .753 & 0.86 & 3.42 \\
\textit{Mean} & \textit{.518} & \textit{58.5} & \textit{61.2} & \textit{93.9} & \textit{.297} & \textit{.725} & \textit{.673} & \textit{.678} & \textit{0.96} & \textit{4.82} \\
\midrule
Mamba2Attn-2.7B & .571 & 29.8 & 13.0 & 100.0 & .170 & .651 & .521 & .712 & 0.98 & 7.12 \\
Zamba2-1.2B & .543 & 76.2 & 51.4 & 100.0 & .414 & .756 & .722 & .774 & 0.15 & 0.90 \\
Qwen3.5-2B & .513 & 89.6 & 80.4 & 96.0 & .459 & .789 & .795 & .782 & 0.09 & 0.52 \\
RecurrentGemma-2B & .484 & 73.4 & 100.0 & 92.7 & .355 & .828 & .840 & .816 & 1.69 & 5.34 \\
Falcon-H1-1.5B & .498 & 16.7 & 1.1 & 100.0 & .083 & .389 & .372 & .408 & 0.90 & 4.78 \\
Nemotron-H-8B & .509 & 75.3 & 30.4 & 100.0 & .383 & .791 & .610 & .346 & 0.25 & 1.28 \\
\textit{Mean} & \textit{.520} & \textit{60.1} & \textit{46.1} & \textit{98.1} & \textit{.311} & \textit{.701} & \textit{.643} & \textit{.640} & \textit{0.68} & \textit{3.32} \\
\bottomrule
\end{tabular}
\caption{GovReport geometry: peak rank, rank and sharpness retention (\%), InterSim final-to-maximum ratio (\%), final rank before and after reinitialization, and native-depth local and global variation.}
\label{tab:extra-govreport}
\end{table*}

\section{Sensitivity Analyses}
\label{appdx:sensitivity-analyses}

\subsection{Matched Token-Embedding Sensitivity}
\label{appdx:matched-token-embedding-sensitivity}

Reinitialization can alter both processing transformations and their token inputs, so we isolate the token-input contribution within a representative pair: Pythia-1.4B (24 blocks) and Mamba-1.4B (48 blocks). For KVPR and MDQA, we use the same 2,000-token task-formatted inputs, Gaussian, Xavier and Kaiming projection reinitializations, and three seeds. Within each model--initializer--seed combination, the two arms share one reinitialized backbone and differ only in whether token embeddings are retained or reinitialized. Normalization, positional and other nonembedding parameters remain fixed; originally tied output heads are separated and preserved. The reinitialized-embedding arm explicitly reinitializes the embedding matrix under Kaiming as well as Gaussian and Xavier.

With the processing weights fixed, we evaluate raw and token-centered spectra using weights proportional to singular values $\sigma_k$ or their energies $\sigma_k^2$, and divide both effective ranks by $\min(n,d)=2{,}000$. We include the input embedding and every native block boundary, with final normalization at the output. For peak and retention comparisons, we first average rank over inputs within each seed, take the peak over layers $l\ge1$, and divide final rank by this peak; we then average the three seed quantities. Final-rank contrasts likewise average paired differences across inputs and seeds.

\begin{table*}[t]
\centering\small
\setlength{\tabcolsep}{4pt}
\begin{tabular}{llrrrr}
\toprule
 & & \multicolumn{2}{c}{Final sigma-linear rank} & \multicolumn{2}{c}{Change in Mamba--Pythia gap (\%)} \\
\cmidrule(lr){3-4}\cmidrule(lr){5-6}
Task & Initializer & Pythia keep/reinit. & Mamba keep/reinit. & Sigma-linear & Energy \\
\midrule
KVPR & Gaussian & 0.112/0.136 & 0.477/0.629 & +35.1 & +227.7 \\
KVPR & Xavier & 0.049/0.073 & 0.321/0.571 & +83.1 & +369.4 \\
KVPR & Kaiming & 0.103/0.139 & 0.672/0.707 & -0.1 & +26.1 \\
MDQA & Gaussian & 0.274/0.299 & 0.619/0.682 & +10.9 & +89.2 \\
MDQA & Xavier & 0.151/0.184 & 0.523/0.628 & +19.5 & +132.0 \\
MDQA & Kaiming & 0.197/0.236 & 0.709/0.718 & -5.8 & +7.2 \\
\bottomrule
\end{tabular}
\caption{Matched embedding effects on raw normalized final rank. Gap changes compare both-reinitialized with both-retained embeddings on the same reinitialized backbones, divided by the both-retained gap. Values are descriptive means over the three seeds and fixed inputs.}
\label{tab:matched-token-embedding-sensitivity}
\end{table*}

Mamba has higher final rank than Pythia under every initializer and seed with either embedding choice, and this ordering also holds after token centering and energy weighting, preserving the reinitialized output-rank direction of Finding~\ref{F2} for this pair. The retained ordering accompanies substantial magnitude sensitivity: reinitializing both embeddings increases the raw sigma-linear gap by $35.1\%$ and $83.1\%$ for Gaussian and Xavier on KVPR, but changes the Kaiming gap by $-0.1\%$ on KVPR and $-5.8\%$ on MDQA (Table~\ref{tab:matched-token-embedding-sensitivity}). The corresponding energy-weighted changes are larger, reaching $369.4\%$ for Xavier on KVPR.

Final rank and retention respond differently to these embedding interventions. Pythia's raw energy-weighted final rank increases in all eighteen task--initializer--seed combinations, while its retention decreases by $8.6$--$20.5$ percentage points because its intermediate peak increases proportionally more strongly. Mamba's MDQA Xavier energy retention instead rises from $44.6\%$ to $94.5\%$, while its Kaiming retention remains $100\%$ with either embedding choice. Token centering further changes the within-model response: Pythia's centered-energy final rank decreases by $0.00082$ for MDQA, Xavier, seed~2. Thus, the representative reinitialized output-rank ordering survives matched embedding policies, while the size of the contrast and its peak-to-output response depend on the embedding intervention and spectral diagnostic.

\subsection{Variation-Norm and Coordinate Sensitivity}
\label{appdx:variation-norm-sensitivity}

\paragraph{Setup.}
We compare GPT-Neo-1.3B, Mamba2-1.3B and Mamba2Attn-2.7B on cached 2,000-token inputs per model on each of KVPR and MDQA. On identical native-layer states, we compare $\rho$ with a joint RMS ratio: the root-mean-square full second difference across layers divided by the root-mean-square deviation from the depth mean, averaging over layers, tokens and features. We also apply two common orthogonal feature rotations; token-centering and alignment comparisons use one full-token input per model and task.

\paragraph{Results.}
The steadiness direction in Finding~\ref{F3} persists in both tasks: Mamba2's mean variation is below GPT-Neo's under both native-state ratios and both tested common rotations. Its ratio to GPT-Neo changes from $0.58$ to $0.82$ on KVPR and from $0.53$ to $0.89$ on MDQA under the joint RMS definition, reducing the separation without reversing it. Mamba2Attn likewise remains below GPT-Neo on both tasks, although its ordering against Mamba2 reverses slightly on MDQA under joint RMS ($0.571$ versus $0.565$). Token centering reverses the GPT-Neo--Mamba2 ordering on the tested KVPR input ($0.949$ versus $0.980$); thus, the prevailing native-state contrast is retained, with preprocessing-dependent exceptions.

\section{Pooling for Location Probes}
\label{appdx:prompts}

The location probes depend on which features are pooled, as Table~\ref{tab:mean-token-probing} illustrates for all four GPT-Neo and Pythia models on KVPR and MDQA. Both poolings use identical inputs and five seeded $70/10/20$ training/validation/testing partitions, with each key--value table or MDQA question kept within one partition and weights selected by validation loss. Under this comparison, last-token pooling yields higher final-layer accuracy for every model on both tasks, supporting the query-conditioned readout used in Section~\ref{sssec:probing-analysis}.

\begin{table}[t]
\centering
\small
\setlength{\tabcolsep}{2.5pt}
\begin{tabular}{l c c c}
\toprule
Model & Pooling & KVPR & MDQA \\
\midrule
\multirow{2}{*}{GPT-Neo-1.3B} & last & $62.1 \da{27.6}$ & $29.3 \da{35.7}$ \\
                               & mean & $54.8 \da{14.4}$ & $26.5 \da{10.9}$ \\
\midrule
\multirow{2}{*}{GPT-Neo-2.7B} & last & $61.1 \da{21.0}$ & $38.9 \da{34.5}$ \\
                               & mean & $57.8 \da{7.6}$ & $23.5 \da{15.9}$ \\
\midrule
\multirow{2}{*}{Pythia-1.4B} & last & $82.8 \da{16.5}$ & $39.4 \da{23.5}$ \\
                               & mean & $71.8 \da{15.7}$ & $26.5 \da{7.2}$ \\
\midrule
\multirow{2}{*}{Pythia-2.8B} & last & $85.9 \da{13.8}$ & $45.1 \da{18.8}$ \\
                               & mean & $76.0 \da{16.8}$ & $28.6 \da{9.5}$ \\
\bottomrule
\end{tabular}
\caption{Final-layer location-probe accuracy (\%), averaged over five seeds, with last-token or mean pooling; $\downarrow$: peak-to-final gap in percentage points.}
\label{tab:mean-token-probing}
\end{table}

\section{Supporting Derivations for the Empirical Analyses}\label{appdx:support}
The empirical analyses distinguish spectral diversity, cross-layer variation and positional accessibility. The following derivations clarify that distinction by establishing the variation ratio's range and calibration, showing how leading and tail directions affect spectral entropy, and identifying the directions read by a fixed location probe.

\subsection{Range of the Layer-Variation Ratio}
The comparison in Table~\ref{tab:variation} uses all $L+1$ native states for all three variation measures: the input embedding at Layer~0 and every computational-layer output, including final normalization where present. We compute the variation measures and their ratio for each input, then average the ratios. In the following bound, $L$ is the model's native computational depth.
\begin{proposition}[Sharp range of relative layer variation]\label{prop:variation-range}
For $L+1$ analyzed layer states with integer $L\ge3$, both $V_{\mathrm{loc}}$ and $V_{\mathrm{glob}}$ lie in $[0,\infty)$. When $V_{\mathrm{glob}}>0$, their ratio has the sharp depth-specific range
\[
\begin{aligned}
0&\le\rho\le C_L\le\sqrt{5},\\
C_L&:=\frac{\sqrt{(L+1)(4L-7)}}{L-1}.
\end{aligned}
\]
The common upper bound $\sqrt{5}$ holds at every integer depth $L\ge3$, and both bounds are preserved when averaging per-input ratios at a fixed depth.
\end{proposition}
\begin{proof}
First consider one coordinate with path $\vu=(u_0,\ldots,u_L)^\top\in\mathbb R^{L+1}$. Center it as $\vy=\vu-\bar u\vone$, where $\bar u=(L+1)^{-1}\sum_{l=0}^{L}u_l$ and $\vone\in\mathbb R^{L+1}$ is the all-ones vector, and let $\mD\in\mathbb R^{(L-1)\times(L+1)}$ have row $l$ equal to the second-difference stencil $(1,-2,1)$ at positions $l-1,l,l+1$, for $l=1,\ldots,L-1$. Since $\mD\vone=\vzero$, this coordinate contributes
\[
a(\vu)=\frac{\|\mD\vy\|_1}{2(L-1)},\qquad
b(\vu)=\frac{\|\vy\|_2}{\sqrt{L+1}}
\]
to local and global variation. For a nonconstant path, $b(\vu)>0$, so
\[
\frac{a(\vu)}{b(\vu)}=\frac{\sqrt{L+1}}{2(L-1)}
\frac{\|\mD\vy\|_1}{\|\vy\|_2}.
\]
To bound the last factor, use $\|\mD\vy\|_1=\max_{\vs\in\{-1,1\}^{L-1}}\vs^\top \mD\vy$ and Cauchy--Schwarz:
\[
\|\mD\vy\|_1\le\max_\vs\|\mD^\top \vs\|_2\,\|\vy\|_2.
\]
\begin{samepage}
The diagonal and first two off-diagonals of $\mD\mD^\top$ are $6$, $-4$ and $1$, respectively. With $m=L-1$,
\begin{align*}
\|\mD^\top \vs\|_2^2
&=6m-8\sum_{i=1}^{m-1}s_is_{i+1}
 +2\sum_{i=1}^{m-2}s_is_{i+2}\\
&\le6m+8(m-1)+2(m-2)\\
&=16m-12=4(4L-7).
\end{align*}
\end{samepage}
Alternating signs $s_i=(-1)^i$ attain this bound. The centered path $\vy=\mD^\top \vs$ also attains equality in Cauchy--Schwarz and has $\operatorname{sign}(\mD\vy)=\vs$, giving $a(\vu)/b(\vu)=C_L$. A nonconstant centered affine path $\vv$ has $\mD\vv=\vzero$, giving the lower endpoint $0$. Since $\vv^\top \mD^\top \vs=0$, their continuous interpolation is never constant and attains every intermediate ratio.

For multiple coordinates, let $u_q^{(l)}$ denote entry $q$ of the vectorized recorded matrix $\mH^{(l)}\in\mathbb R^{n\times d}$ and define its depth path $\vu_q=(u_q^{(0)},\ldots,u_q^{(L)})^\top\in\mathbb R^{L+1}$, $q=1,\ldots,nd$. Apply $a(\vu_q)\le C_L b(\vu_q)$ to each coordinate and sum. Dividing by $nd$ gives $V_{\mathrm{loc}}\le C_LV_{\mathrm{glob}}$, and averaging per-input ratios preserves the interval. The common upper bound follows from $5-C_L^2=(L-3)(L-4)/(L-1)^2\ge0$ for integer $L\ge3$. Both measures are nonnegative and scale by $|c|$ under $\vu_q\mapsto c\vu_q$, so each is unbounded while their ratio is scale-invariant. If $V_{\mathrm{glob}}=0$, every coordinate is constant across depth and $V_{\mathrm{loc}}=0$ too; the ratio is then undefined, not zero.
\end{proof}
The ratio measures local variation relative to global dispersion and can exceed one.

\subsection{Layer Variation in Terms of Increments}
\begin{proposition}[Increment form of layer variation]\label{prop:increments}
For one coordinate with values $u_0,\ldots,u_L$ ($L\ge2$) across layers, write its increments $d_a=u_a-u_{a-1}$, $a=1,\ldots,L$, and stack them as $\vd=(d_1,\ldots,d_L)^\top\in\mathbb R^L$. Then, for $l=1,\ldots,L-1$ in the first identity,
\begin{align*}
\Big|u_l-\tfrac{u_{l-1}+u_{l+1}}2\Big|&=\tfrac12|d_{l+1}-d_l|,\\
\frac1{L+1}\sum_{l=0}^{L}(u_l-\bar u)^2&=\vd^\top \mG\vd,
\end{align*}
where $\mG\in\mathbb R^{L\times L}$ has scalar entries, for $a,b\in\{1,\ldots,L\}$,
\[
G_{a,b}=\min(a,b)\,[L+1-\max(a,b)]/(L+1)^2.
\]
\end{proposition}
\begin{proof}
The midpoint residual equals $(d_l-d_{l+1})/2$. For $M=L+1$ values, $\sum_{i<j}(u_j-u_i)^2=M\sum_l(u_l-\bar u)^2$, and $u_j-u_i=\sum_{a=i+1}^{j}d_a$. Expanding the square, the coefficient of $d_ad_b$ counts pairs $i<j$ spanning both $a$ and $b$: there are $\min(a,b)$ choices of $i$ and $M-\max(a,b)$ choices of $j$. Dividing by $M^2$ gives $\mG$.
\end{proof}
$V_{\mathrm{loc}}$ thus depends only on how consecutive increments differ, whereas $V_{\mathrm{glob}}$ accumulates increments through the positive entries of $\mG$. A path with a nonzero constant increment $v$ has $V_{\mathrm{loc}}=0$ and depth standard deviation $|v|\sqrt{L(L+2)/12}$, so $\rho=0$.

For a random-walk reference, take independent random increments $\rd_a\sim\mathcal N(0,s^2)$, $a=1,\ldots,L$, with $s>0$, and stack them as $\rvd=(\rd_1,\ldots,\rd_L)^\top\in\mathbb R^L$. Let $\rvu=(\ru_0,\ldots,\ru_L)^\top\in\mathbb R^{L+1}$ be the random path with fixed initial value $\ru_0=u_0$ and $\ru_l=u_0+\sum_{a=1}^{l}\rd_a$. A realization of this reference is written as $\vu$ with increments $\vd$; $a$ and $b$ remain the same deterministic contribution functions. The Gaussian-step model supplies a reference calibration for the recorded representations. Then $\mathbb E|\rd_{l+1}-\rd_l|/2=s/\sqrt\pi$ and $\mathbb E[\rvd^\top \mG\rvd]=s^2L(L+2)/[6(L+1)]$. With many independent coordinates, the coordinate-averaged measures concentrate around these paths' expected local variation and expected depth standard deviation. Jensen's inequality gives
\[
\frac{\mathbb E a(\rvu)}{\mathbb E b(\rvu)}
\ge\sqrt{\frac{6(L+1)}{\pi L(L+2)}}.
\]
This bound concerns the ratio of expectations under the reference model. The reference at each native depth is determined by its $L+1$ analyzed states and the independent Gaussian-step model.

\subsection{Leading--Tail Spectral Entropy}
\begin{proposition}[Leading--tail decomposition]\label{prop:tail}
For singular values $\sigma_k>0$ and $\beta>0$, let $p_k^{(\beta)}=\sigma_k^\beta/\sum_{k'}\sigma_{k'}^\beta$ and $\mathcal H_\beta=-\sum_kp_k^{(\beta)}\log p_k^{(\beta)}$, so that $\beta=1$ gives the entropy behind effective rank and $\beta=2$ the energy-weighted entropy. For a partition of the singular-value indices into nonempty groups of leading directions $\mathbb A$ and tail directions $\mathbb B$, with leading mass $a_\beta=\sum_{k\in\mathbb A}p_k^{(\beta)}$ and conditional entropies $\mathcal H_{\mathbb A,\beta},\mathcal H_{\mathbb B,\beta}$ within each group,
\[
\mathcal H_\beta=h_{\mathrm{bin}}(a_\beta)+a_\beta\mathcal H_{\mathbb A,\beta}+(1-a_\beta)\mathcal H_{\mathbb B,\beta},
\]
where $h_{\mathrm{bin}}$ is the binary entropy. For one leading value $x\,b$ above $M$ equal tail values $b$, $a_\beta=x^\beta/(x^\beta+M)$ and $d\mathcal H_\beta/d\log x=-\beta^2a_\beta(1-a_\beta)\log x$.
\end{proposition}
\begin{proof}
For fixed $\beta$, define $p_{k\mid\mathbb A}^{(\beta)}=p_k^{(\beta)}/a_\beta$ for $k\in\mathbb A$ and $p_{k\mid\mathbb B}^{(\beta)}=p_k^{(\beta)}/(1-a_\beta)$ for $k\in\mathbb B$. Write $p_k^{(\beta)}=a_\beta p_{k\mid\mathbb A}^{(\beta)}$ on $\mathbb A$ and $p_k^{(\beta)}=(1-a_\beta)p_{k\mid\mathbb B}^{(\beta)}$ on $\mathbb B$ and expand the logarithms. For one leader, $\mathcal H_\beta=h_{\mathrm{bin}}(a_\beta)+(1-a_\beta)\log M$, whose derivative follows from $da_\beta/d\log x=\beta a_\beta(1-a_\beta)$.
\end{proof}
Both entropies thus weigh the leading mass against the spread of the tail, but squaring gives the leading directions more mass. Concentration in a few leading directions can therefore lower the energy-weighted entropy while a broadening tail keeps or raises the effective rank. For example, changing the spectrum $(10,2,0.1,\ldots,0.1)$ with eight tail values to $(12,2,0.2,\ldots,0.2)$ raises the effective rank from $2.19$ to $2.49$ while lowering the energy-weighted effective rank from $1.19$ to $1.15$, the combination observed at hybrid attention blocks (Appendix~\ref{appdx:hybrid-accessibility}).

\subsection{Directions Read by a Location Probe}
\begin{proposition}[Score subspace of a linear probe]\label{prop:score}
Let a probe with feature vector $\vx\in\mathbb R^d$, weight vectors $\vw_j\in\mathbb R^d$ and scalar biases $b_j$, $j=1,\ldots,K$, assign scores $f_j(\vx)=\vw_j^\top \vx+b_j$ to $K$ positions and predict $\arg\max_jf_j(\vx)$. Let $\bm{\Pi}\in\mathbb R^{d\times d}$ project orthogonally onto $\mathbb S=\operatorname{span}\{\vw_j-\vw_K:1\le j<K\}$. Then $\dim \mathbb S\le K-1$ and $f_j(\bm{\Pi}\vx)-f_k(\bm{\Pi}\vx)=f_j(\vx)-f_k(\vx)$ for all $\vx$ and all $j,k$. Thus, the probe predicts the same position for $\bm{\Pi}\vx$ and $\vx$.
\end{proposition}
\begin{proof}
With $\mI\in\mathbb R^{d\times d}$ the identity, each $\vw_j-\vw_K$ lies in $\mathbb S$, so $(\vw_j-\vw_K)^\top(\mI-\bm{\Pi})=\vzero^\top$, and subtracting two such identities gives $(\vw_j-\vw_k)^\top(\mI-\bm{\Pi})\vx=0$. All score differences, and hence the maximizing position, are unchanged. $\mathbb S$ is spanned by $K-1$ vectors.
\end{proof}
A probe that distinguishes 11 KVPR positions or seven MDQA target positions therefore reads at most ten or six directions of a representation of width $d\ge1{,}024$, with all predictions unchanged whenever these score-difference directions are preserved and the same tie rule is used. Training standardization is absorbed into the raw affine weights and biases, with $\mA\in\mathbb R^{d\times d}$, $\vc\in\mathbb R^d$, $\mW,\widetilde{\mW}\in\mathbb R^{K\times d}$ and $\vb,\widetilde{\vb}\in\mathbb R^K$ (row $j$ of $\mW$ is $\vw_j^\top$): if $\vz=\mA\vx+\vc$ and the fitted scores are $\widetilde{\mW}\vz+\widetilde{\vb}$, then $\mW=\widetilde{\mW}\mA$ and $\vb=\widetilde{\mW}\vc+\widetilde{\vb}$. This statement applies to each fixed probe separately, with its subspace determined by its fitted weights at a given layer and seed. The subspace concerns last-token features across inputs, whereas geometry rank summarizes the full token spectrum within each input.

\section{Attention--Mamba2 Case Studies and Proofs}
\label{appdx:theory-proofs}
This appendix derives the three case studies in Section~\ref{sec:theoretical-analysis} from the attention and SSD blocks in Mamba2Attn. Complete raw residual states satisfy $\mH_{\mathrm{raw}}^{(l+1)}=\mH_{\mathrm{raw}}^{(l)}+F_l(\mH_{\mathrm{raw}}^{(l)})$; $F_l:\mathbb R^{n\times d}\to\mathbb R^{n\times d}$ is the complete branch update, including internal normalization, mixing, gating and output projection but excluding the terminal normalization. The recorded states satisfy $\mH^{(l)}=\mH_{\mathrm{raw}}^{(l)}$ for $l<L$ and $\mH^{(L)}=\mathcal N_f(\mH_{\mathrm{raw}}^{(L)})$, where $\mathcal N_f$ is the model's terminal normalization map, or the identity when absent. The active-SSD construction below supplies Proposition~\ref{prop:active-main}, and the complete-path survival bound proves Proposition~\ref{prop:ssd-survival-main}.

\begin{figure*}[t]
\centering
\includegraphics[width=\textwidth]{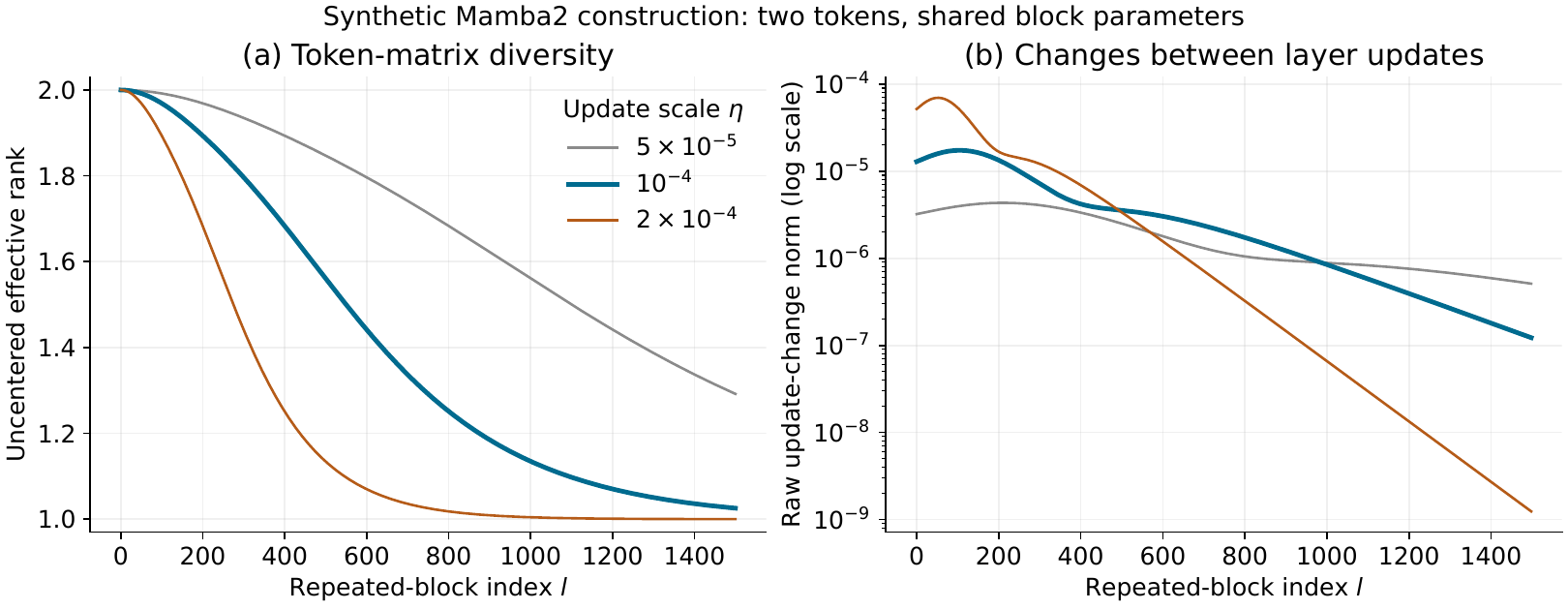}
\caption{Numerical illustration of diversity loss through repeated Mamba2 updates. A synthetic two-token input starts with coefficients $\bm{\xi}=(1,-1)^\top$ and passes through 1,500 copies of the complete block in Eq.~\ref{eq:active}, at three update scales $\eta$. (a) Effective rank, defined in Eq.~\ref{eq:erank}, approaches one as the token-specific component disappears. (b) The raw update-change norm $\|\mH_{\mathrm{raw}}^{(l+2)}-2\mH_{\mathrm{raw}}^{(l+1)}+\mH_{\mathrm{raw}}^{(l)}\|_F$ is shown on a logarithmic scale. This raw update-change norm measures the unnormalized second difference; the empirical ratio $\rho$ additionally divides coordinate-averaged local variation by global depth dispersion. All curves use the same block count, so changing $\eta$ also changes the accumulated update amount. These are computed trajectories of the construction; Theorem~\ref{thm:active} guarantees convergence locally.}
\label{fig:theory-erank}
\end{figure*}

\paragraph{Numerical setup.}
The illustration uses $c=1$, $d=2560$, $m=5120$, $N=128$, $p=64$, $a=0.8$, $\Delta=D=1$, unit normalization gains and $\epsilon_p=\epsilon_g=10^{-5}$, with $\eta\in\{5\times10^{-5},10^{-4},2\times10^{-4}\}$. All other value channels are zero. Effective rank is computed from the raw, uncentered token matrix, before terminal normalization and at every block boundary.

\subsection{Part I: what attention and Mamba2 can transport}
\subsubsection{Use the actual Mamba2 operation}
For an input matrix $\mX\in\mathbb R^{n\times d}$, let $\widehat{\mX}=\mathrm{RMSNorm}_{\rm pre}(\mX)$. An input projection produces token-row streams $\mZ,\mX^{\rm raw},\mB^{\rm raw},\mC^{\rm raw}$ and raw step sizes $\Delta_t^{{\rm raw},h}$. A causal depthwise convolution and SiLU produce the value, write and read vectors $\vx_t^h,\vb_t^{g(h)},\vc_t^{g(h)}$, but do not act on $\mZ$ or the step sizes. With $n_{\rm head}$ heads of width $p$, state size $N$, and shared group index $g(h)$, the SSD recurrence is
\begin{align}
 \Delta_t^h&=\operatorname{softplus}(\Delta_t^{{\rm raw},h}+b_\Delta^h),\\
 a_t^h&=e^{-\lambda_h\Delta_t^h},\quad \lambda_h>0,\\
 \mS_t^h&=a_t^h\mS_{t-1}^h+\Delta_t^h \vb_t^{g(h)}(\vx_t^h)^\top,\\
 \vy_t^h&=(\mS_t^h)^\top \vc_t^{g(h)}+D_h\vx_t^h.
 \label{eq:ssd}
\end{align}
Here $\mS_t^h\in\mathbb R^{N\times p}$, $\vx_t^h,\vy_t^h\in\mathbb R^p$, and $\vb_t^{g(h)},\vc_t^{g(h)}\in\mathbb R^N$. One scalar $a_t^h$ multiplies all state and channel coordinates in a head, in contrast to S6's independently decaying state coordinates~\citep{mamba2}. We use zero initial state and a single uninterrupted, unpadded input.

In Mamba2Attn, the write/read vectors $\vb_t,\vc_t$ are shared across heads, the SiLU gate precedes group RMSNorm, and the mixer has no separate MLP:
\begin{align}
 \vw_t&=\vy_t^{\rm all}\odot\operatorname{SiLU}(\vz_t),\\
 F_M(\mX)_{t,:}&=\left[\bm{\gamma}\odot
 \frac{\vw_t}{\sqrt{\scnorm{\vw_t}_2^2/m+\epsilon_g}}\right]^\top \mW_o.
 \label{eq:full}
\end{align}
Here $\vy_t^{\rm all}=\operatorname{concat}_h\vy_t^h$ collects the head outputs. Equation~\ref{eq:full} shows the single-group case $m=n_{\rm head}p$, with $\vy_t^{\rm all},\vz_t,\vw_t,\bm{\gamma}\in\mathbb R^m$ and $\mW_o\in\mathbb R^{m\times d}$; when there are multiple groups, the bracketed operation is applied within each group and their outputs are concatenated before the output projection. Input prenorm, the $D$ skip, output projection and residual addition remain part of the computation. Gate order matters: attenuation before RMSNorm need not become attenuation afterward.

\subsubsection{Independent temporal routes through an SSD head}
Conditioning on realized $\vb_t^{g(h)},\vc_t^{g(h)},\Delta_t^h$ and unrolling Eq.~\ref{eq:ssd} gives
\begin{align}
 \vy_t^h&=\sum_{j\le t}T_{t,j}^h\vx_j^h+D_h\vx_t^h,\\
 T_{t,j}^h&=\Delta_j^h\langle \vc_t^{g(h)},\vb_j^{g(h)}\rangle
             \prod_{u=j+1}^{t}a_u^h.
 \label{eq:kernel}
\end{align}
The matrix $\mT^h$ with scalar entries $T_{t,j}^h$ is the conditional transfer with realized coefficients held fixed. Differentiating the full input-dependent map additionally differentiates these coefficients. We suppress the head and sharing-group indices for a single head below.

\begin{proposition}[Cross-cut SSD constraint]\label{prop:cut}
For any token cut $c$, the strict earlier-to-later block $\mT_{t>c,j\le c}$ of one head has rank at most $N$.
\end{proposition}
\begin{proof}
Split the product at $c$. The block factors as $\mL\mR^\top$, with rows $\vl_t^\top$ and $\vr_j^\top$, respectively, where
\begin{align*}
 \vl_t&=\left(\prod_{u=c+1}^{t}a_u\right)\vc_t,\\
 \vr_j&=\left(\prod_{u=j+1}^{c}a_u\right)\Delta_j \vb_j.
\end{align*}
The vectors $\vl_t,\vr_j\in\mathbb R^N$ therefore give both factor matrices $N$ columns. The local $D$ skip has no strict cross-cut entry.
\end{proof}
The cross-cut rank counts independently addressable temporal payload modes in a conditional head. The full triangular $\mT+D\mI$ can be full rank, and Mamba2Attn's $N=128$ exceeds head width $p=64$, leaving room for a head's full feature rank within the cut constraint. Whole-model effective rank additionally reflects gates, multiple heads and residual connections, as well as the token features being transported.

\subsubsection{Attention can preserve independently selected payloads}
For the corresponding attention block, normalized inputs produce convolved Q/K/V streams, with no SiLU in this QKV convolution. Rotary maps act on part of Q/K. For width-$p$ heads, set $s_h=p^{-1/2}$; with the head index suppressed on the query and key vectors,
\begin{align}
 P^h_{t,j}&=\frac{\exp(s_h \vq_t^\top \vk_j)}{\sum_{u\le t}\exp(s_h \vq_t^\top \vk_u)}\mathbf1_{j\le t},\\
 F_A(\mX)&=\operatorname{concat}_h(\mP^h\mV^h)\mW_o.
 \label{eq:attn}
\end{align}
Here $\mV^h$ stacks $(\vv_j^h)^\top$ as rows, $\mP^h$ has scalar entries $P^h_{t,j}$, and $\mathbf1_{j\le t}$ is the scalar causal indicator. Query-specific routing lets each later token select among the values available in the current residual stream.

\begin{proposition}[Independent routes from actual softmax]\label{prop:attn}
For any finite $r\ge2$, an attention head with two non-rotary query/key coordinates can produce an $r\times r$ strict cross-cut submatrix arbitrarily close to $\mI_r$.
\end{proposition}
\begin{proof}
Let $\mathbb J=\{1,\ldots,r\}$ and $\mathbb I=\{r+1,\ldots,2r\}$ be the prefix and suffix index sets, each in increasing order. Choose $r$ distinct unit-circle directions $\vu_i\in\mathbb R^2$ as prefix keys. Let the query at suffix token $r+i$ be $(\beta/s_h)\vu_i$, so $\beta>0$ is the effective logit scale after absorbing $s_h$, and set suffix keys to zero. For
\[
 \delta=\min\left(1,\min_{i\ne j}(1-\vu_i^\top\vu_j)\right)>0,
\]
the desired key has margin at least $\beta\delta$ over every competitor. Each suffix row has at most $2r-1$ competitors, so its undesired probability is at most $(2r-1)e^{-\beta\delta}$. Consequently,
\begin{equation}
 \scnorm{\mP_{\mathbb I,\mathbb J}-\mI_r}_2\le2\sqrt r(2r-1)e^{-\beta\delta}=:\varepsilon_\beta.
\end{equation}
Thus $\sigma_{\min}(\mP_{\mathbb I,\mathbb J})\ge1-\varepsilon_\beta$. Taking $\beta$ sufficiently large proves the claim. Shared projections realize these keys and queries by reserving separate input coordinates for their content; filler coordinates give equal input norms, and only the current convolution tap need be nonzero. The Mamba2Attn attention head leaves non-rotary coordinates available.
\end{proof}
For $r>N$, these routes exceed the single-head SSD constraint.

If a value matrix $\mV_{\mathbb J}$ stacks independent prefix payload rows in this order, then
\begin{equation}
 \scnorm{\mP_{\mathbb I,\mathbb J}\mV_{\mathbb J}-\mV_{\mathbb J}}_F\le\varepsilon_\beta\scnorm{\mV_{\mathbb J}}_F.
 \label{eq:payload}
\end{equation}
With zero competing values and an isometric output map on the payload subspace, the attention output preserves their singular directions up to this error. This supplies an architectural route to retention.

\paragraph{Full nonlinear influence has a different bound.}
If suffix block inputs are fixed while prefix inputs vary, Mamba2 communicates across the cut through its SSM state and convolution cache. With $n_{\rm head}$ heads, total SSD value width $d_{\rm ssm}=n_{\rm head}p$, convolution width $w$ and $g$ write/read sharing groups indexed by $g(h)$, the flattened prefix-to-suffix Jacobian has rank at most
\begin{equation}
 n_{\rm head}Np+(w-1)(d_{\rm ssm}+2gN).
\end{equation}
This follows by factoring through those state coordinates. It includes input-dependent coefficients.

\subsection{Part II: steady diversity loss inside a complete Mamba2 block}
\subsubsection{An active-SSD construction}
Consider two tokens of width $d$, with feature-space standard basis vectors $\ve^{(1)},\ve^{(2)}\in\mathbb R^d$ and state-space standard basis vector $\ve_{\mathrm{state}}^{(1)}\in\mathbb R^N$. For $c>0$ and $\bm{\xi}\in\mathbb R^2$, define
\begin{equation}
 \mH_{\mathrm{raw}}(\bm{\xi})_{t,:}=(c\ve^{(1)}+\xi_t\ve^{(2)})^\top.
 \label{eq:two}
\end{equation}
Use unit prenorm and group-norm gains, positive norm epsilons, one group of width $m=d_{\rm ssm}=2d$, current-only convolution taps and zero convolution biases. Choose shared input projections so that
\begin{align}
 r_t&=\sqrt{(c^2+\xi_t^2)/d+\epsilon_p},\\
 u_t&=c/r_t,\quad g_t=\operatorname{SiLU}(u_t),\\
 \vb_t&=\vc_t=g_t\ve_{\mathrm{state}}^{(1)},\quad \operatorname{SiLU}(\vz_t)=g_t\vone.
\end{align}
Choose an active head $h_*$, ordered first in the concatenation. The first three entries of $\vx_t^{h_*}\in\mathbb R^p$ are $\operatorname{SiLU}(\xi_t/r_t)$, $\operatorname{SiLU}(-\xi_t/r_t)$ and $\operatorname{SiLU}(u_t)$; all remaining entries are zero.
All other heads have zero value inputs and hence zero outputs, so $\vy_t^{\rm all}=\operatorname{concat}_h\vy_t^h\in\mathbb R^m$ is the active head output padded with zeros. Set $\Delta=1$, $a\in(0,1)$ and $D=1$; unused state modes and value channels have zero writes. These choices are realizable by the shared projections and scalar transition parameterization, including the Mamba2Attn dimensions $d=2560$, $N=128$, $p=64$. The state carries continuously from the first token to the second.

Run Eq.~\ref{eq:ssd}, obtaining $\vy_t^{h_*}=(\mS_t^{h_*})^\top\vc_t+\vx_t^{h_*}$ with $\mS_t^{h_*}\in\mathbb R^{N\times p}$. Gate and normalize $\vy_t^{\rm all}$ exactly as Eq.~\ref{eq:full}, and set the output map's only nonzero feature coordinate to minus $\eta$ times the difference of its first two channels. Writing $\vf=(f_1,f_2)^\top$, the complete residual update is
\begin{align}
 \bm{\xi}^+&=\bm{\xi}-\eta \vf(\bm{\xi}),\\
 f_t(\bm{\xi})&=\frac{g_t(y_{t,1}^{\rm all}-y_{t,2}^{\rm all})}
 {\sqrt{g_t^2\scnorm{\vy_t^{\rm all}}_2^2/m+\epsilon_g}}.
 \label{eq:active}
\end{align}
Here $y_{t,i}^{\rm all}$ is the scalar $i$th coordinate of $\vy_t^{\rm all}$, and repeated updates give $\mH_{\mathrm{raw}}^{(l)}=\mH_{\mathrm{raw}}(\bm{\xi}^{(l)})$. The recurrence is active: $\vb_t,\vc_t$ and the third value channel are nonzero, and token two receives token one's state. Its output can change the contrast coordinate while leaving the common coordinate $c$ unchanged.

\begin{theorem}[Local diversity loss with small update drift]\label{thm:active}
For the complete tied block in Eq.~\ref{eq:active}, there exists a sufficiently small $\eta>0$ for which $\bm{\xi}=\vzero$ is locally asymptotically stable. Along any convergent trajectory with initially independent token rows, $\mathrm{eRank}(\mH_{\mathrm{raw}}(\bm{\xi}^{(l)}))\to1$. On a compact neighborhood containing the trajectory, complete-state second differences are $O(\eta^2)$.
\end{theorem}
\begin{proof}
At $\bm{\xi}=\vzero$, let $r_0=\sqrt{c^2/d+\epsilon_p}$, $g_0=\operatorname{SiLU}(c/r_0)$ and $\kappa=g_0^2>0$. The first two value channels vanish, so $\vf(\vzero)=\vzero$. Derivatives of $r_t,g_t,\vb_t,\vc_t$ with respect to $\bm{\xi}$ vanish at zero. Using $\operatorname{SiLU}(x)-\operatorname{SiLU}(-x)=x$, set $d_t=g_0/(r_0\nu_t)$. Differentiation then gives
\begin{equation}
 D\vf(\vzero)=\operatorname{diag}(d_1,d_2)
 \begin{pmatrix}\kappa+1&0\\a\kappa&\kappa+1\end{pmatrix}=:\mJ_0,
\end{equation}
where $\nu_t=\sqrt{g_0^2\scnorm{\vy_t^{\rm all}(\vzero)}_2^2/m+\epsilon_g}>0$. This lower-triangular matrix has positive diagonal. For $0<\eta<2/\max_t(J_0)_{t,t}$, all eigenvalues of $\mI-\eta \mJ_0$ have modulus below one. Smooth discrete-time local stability gives $\bm{\xi}^{(l)}\to\vzero$ for a neighborhood of initial states. Continuity of effective rank at the nonzero limit $c\vone (\ve^{(1)})^\top$ gives the spectral limit. Finally, if $\vf$ is $L_f$-Lipschitz and bounded by $M_f$ on the neighborhood,
\begin{align*}
 &\scnorm{\bm{\xi}^{(l+2)}-2\bm{\xi}^{(l+1)}+\bm{\xi}^{(l)}}_2\\
 &\quad=\eta\scnorm{\vf(\bm{\xi}^{(l+1)})-\vf(\bm{\xi}^{(l)})}_2\\
 &\quad\le\eta^2L_fM_f.
\end{align*}
The matrix state has exactly this varying coordinate, proving the drift bound.
\end{proof}

\subsubsection{What the construction proves about steadiness}
Theorem~\ref{thm:active} derives a coexistence mechanism from an actual Mamba2 parameter assignment. Its chosen output projection removes token contrast through repeated, small complete updates, even with a nonzero selective recurrence.

The bound concerns raw adjacent blocks. At fixed depth, sending $\eta$ to zero suppresses total change, so substantial diversity loss requires enough accumulated updates. Translating the second-difference bound to $\rho$ also requires controlling global depth dispersion. The recorded native-depth path includes final normalization and has its own increments and dispersion, to which the variation definitions apply directly.

\subsection{Part III: what Mamba2Attn can combine}
\subsubsection{Attention supplies a read; Mamba2 determines its survival}
At an attention block, Eq.~\ref{eq:payload} can route distinct payloads from earlier residual tokens into later query tokens, exposing information retained at those earlier positions. This read acts on the values available in the current residual stream; its spectral contribution depends on their directions and scale.

To follow an attention update through later blocks, compare two raw streams $\mH_{\mathrm{raw}}^{(l)}$ and $\widetilde{\mH}_{\mathrm{raw}}^{(l)}$ just after the intervention at boundary $a+1$. Set $\bm{\delta}^{(l)}=\operatorname{vec}(\mH_{\mathrm{raw}}^{(l)}-\widetilde{\mH}_{\mathrm{raw}}^{(l)})$, where $\operatorname{vec}$ stacks all token and feature entries into one vector. The streams share $k=b-a-1$ subsequent residual maps $G_l(\mX)=\mX+F_l(\mX)$, indexed by $l=a+1,\ldots,b-1$. Thus $\bm{\delta}_{\rm in}=\bm{\delta}^{(a+1)}$ and $\bm{\delta}_{\rm out}=\bm{\delta}^{(b)}$; the Euclidean norms here equal the matrix Frobenius norms in Eq.~\ref{eq:ssd-survival-main}. In particular, $\bm{\delta}_{\rm in/out}=\operatorname{vec}(\bm{\Delta}_{\rm in/out})$ for the matrix differences in that equation. Define the vector-coordinate branch $F_l^{\rm vec}(\vx)=\operatorname{vec}(F_l(\operatorname{unvec}(\vx)))$, where $\operatorname{unvec}:\mathbb R^{nd}\to\mathbb R^{n\times d}$ inverts the fixed stacking order. On each segment between the streams, define the $nd\times nd$ Jacobian
\begin{equation}
 \bar{\mJ}_l=\int_0^1 DF_l^{\rm vec}(\operatorname{vec}(\widetilde{\mH}_{\mathrm{raw}}^{(l)})+s\bm{\delta}^{(l)})\,ds.
\end{equation}
With $\mI=\mI_{nd}$, the fundamental theorem of calculus gives the exact finite-difference relation
\begin{align}
 \bm{\delta}^{(l+1)}&=(\mI+\bar{\mJ}_l)\bm{\delta}^{(l)},\\
 \bm{\delta}^{(b)}&=(\mI+\bar{\mJ}_{b-1})\cdots(\mI+\bar{\mJ}_{a+1})\bm{\delta}_{\rm in}.
 \label{eq:survive}
\end{align}
These complete-map Jacobians include changes in selective coefficients, convolution, gates and normalization.

\paragraph{Vector form and proof of Proposition~\ref{prop:ssd-survival-main}.}
In vector coordinates, the bounds are
\begin{align}
 \scnorm{\bm{\delta}^{(b)}}_2&\ge\scnorm{\bm{\delta}_{\rm in}}_2\prod_{l=a+1}^{b-1}(1-s_l),\\
 \scnorm{\bm{\delta}^{(b)}}_2&\le\scnorm{\bm{\delta}_{\rm in}}_2\prod_{l=a+1}^{b-1}(1+s_l).
\end{align}
\begin{proof}
The branch difference is bounded by $s_l\scnorm{\bm{\delta}^{(l)}}_2$. The triangle and reverse triangle inequalities applied to $\bm{\delta}^{(l)}+\operatorname{vec}\bigl(F_l(\mH_{\mathrm{raw}}^{(l)})-F_l(\widetilde{\mH}_{\mathrm{raw}}^{(l)})\bigr)$ give the one-step bounds; iteration gives the products.
\end{proof}
For the concrete block in Eq.~\ref{eq:active}, the branch is $-\eta\vf$ on its contrast coordinates, giving the explicit lower survival factor $(1-\eta L_f)^k$ over $k$ blocks whenever $\eta L_f<1$ on both paths and their connecting segments.

\subsubsection{Preserved attention updates and spectral diversity}
Equation~\ref{eq:survive} bounds the size of a transported difference. Its contribution to spectral diversity depends on its alignment and scale: a payload outside dominant reference directions can broaden the spectrum, whereas an update concentrated in an already dominant direction can lower effective rank. Final normalization further transforms the resulting spectrum. When the transported difference preserves linear class separation in last-token features after final normalization, stronger answer-location accessibility can coexist with lower spectral retention, as in the Mamba2Attn--Mamba2 comparison.

The six actual attention placements determine distinct downstream transport paths. Between successive attention blocks there are 8, 8, 8, 8 and 10 Mamba2 blocks, and seven after the last attention block. The survival product therefore depends on the placement and complete-map sensitivity of each intervening block, while differences between consecutive complete updates determine layer-wise variation.

\section{Use of AI Assistants}
ChatGPT, Gemini, and Claude Code were used to assist with debugging experimental code. The authors reviewed and verified the resulting code and take full responsibility for the paper.

\end{document}